\documentclass{article}

\PassOptionsToPackage{numbers, compress}{natbib}
\usepackage[preprint]{neurips_2026}

\usepackage[utf8]{inputenc} 
\usepackage[T1]{fontenc}    
\PassOptionsToPackage{hyphens}{url}
\usepackage{hyperref}       
\usepackage{url}            
\usepackage{booktabs}       
\usepackage{amsfonts}       
\usepackage{nicefrac}       
\usepackage{microtype}      
\usepackage[dvipsnames]{xcolor}
\usepackage{amsmath}
\usepackage{graphicx}
\graphicspath{{./}}
\usepackage{csquotes}
\usepackage{amsthm}
\usepackage{wrapfig}
\usepackage{subcaption}
\usepackage{dirtytalk}
\usepackage{multirow}
\usepackage{mathtools}
\usepackage{array}
\usepackage[inline]{enumitem}
\usepackage{tabularray}
\usepackage[most]{tcolorbox}
\usepackage{float}
\usepackage{titletoc}
\usepackage{algorithm}
\usepackage{algpseudocode}

\usepackage{tikz}
\usepackage{booktabs}

\newcommand{\rtpic}[2]{%
\raisebox{2pt}{%
\begin{tikzpicture}[
  baseline={(current bounding box.center)},
  x=0.55ex,
  y=0.55ex,
  line cap=round
]
#1
\foreach \p in {#2}{\filldraw (\p) circle [radius=1.35pt];}
\end{tikzpicture}%
}
}%

\newcommand{\TEmpty}{\ensuremath{\emptyset}}

\newcommand{\TLeaf}{%
\rtpic{\coordinate (r) at (0,0);}{r}}

\newcommand{\TChainTwo}{%
\rtpic{
  \coordinate (r) at (0,0);
  \coordinate (a) at (0,2);
  \draw[thin] (r)--(a);
}{r,a}}

\newcommand{\TTwoLeaves}{%
\rtpic{
  \coordinate (r) at (0,0);
  \coordinate (a) at (-1.5,2);
  \coordinate (b) at (1.5,2);
  \draw[thin] (a)--(r)--(b);
}{r,a,b}}

\newcommand{\TChainThree}{%
\rtpic{
  \coordinate (r) at (0,0);
  \coordinate (a) at (0,2);
  \coordinate (b) at (0,4);
  \draw[thin] (r)--(a)--(b);
}{r,a,b}}

\newcommand{\TThreeLeaves}{%
\rtpic{
  \coordinate (r) at (0,0);
  \coordinate (a) at (-2,2);
  \coordinate (b) at (0,2);
  \coordinate (c) at (2,2);
  \draw[thin] (a)--(r)--(b);
  \draw[thin] (r)--(c);
}{r,a,b,c}}

\newcommand{\TLeafChainTwo}{%
\rtpic{
  \coordinate (r) at (0,0);
  \coordinate (a) at (-1.5,2);
  \coordinate (b) at (1.5,2);
  \coordinate (c) at (1.5,4);
  \draw[thin] (a)--(r)--(b)--(c);
}{r,a,b,c}}

\newcommand{\TChainTwoLeaf}{%
\rtpic{
  \coordinate (r) at (0,0);
  \coordinate (a) at (-1.5,2);
  \coordinate (b) at (-1.5,4);
  \coordinate (c) at (1.5,2);
  \draw[thin] (b)--(a)--(r)--(c);
}{r,a,b,c}}

\newcommand{\TNestedTwoLeaves}{%
\rtpic{
  \coordinate (r) at (0,0);
  \coordinate (a) at (0,2);
  \coordinate (b) at (-1.5,4);
  \coordinate (c) at (1.5,4);
  \draw[thin] (r)--(a);
  \draw[thin] (b)--(a)--(c);
}{r,a,b,c}}

\newcommand{\TChainFour}{%
\rtpic{
  \coordinate (r) at (0,0);
  \coordinate (a) at (0,2);
  \coordinate (b) at (0,4);
  \coordinate (c) at (0,6);
  \draw[thin] (r)--(a)--(b)--(c);
}{r,a,b,c}}

\newcommand{\TFourLeaves}{%
\rtpic{
  \coordinate (r) at (0,0);
  \coordinate (a) at (-3,2);
  \coordinate (b) at (-1,2);
  \coordinate (c) at (1,2);
  \coordinate (d) at (3,2);
  \draw[thin] (a)--(r)--(b);
  \draw[thin] (r)--(c);
  \draw[thin] (r)--(d);
}{r,a,b,c,d}}

\newcommand{\TLeafLeafChainTwo}{%
\rtpic{
  \coordinate (r) at (0,0);
  \coordinate (a) at (-2.2,2);
  \coordinate (b) at (0,2);
  \coordinate (c) at (2.2,2);
  \coordinate (d) at (2.2,4);
  \draw[thin] (a)--(r)--(b);
  \draw[thin] (r)--(c)--(d);
}{r,a,b,c,d}}

\newcommand{\TLeafChainTwoLeaf}{%
\rtpic{
  \coordinate (r) at (0,0);
  \coordinate (a) at (-2.2,2);
  \coordinate (b) at (0,2);
  \coordinate (c) at (0,4);
  \coordinate (d) at (2.2,2);
  \draw[thin] (a)--(r)--(b)--(c);
  \draw[thin] (r)--(d);
}{r,a,b,c,d}}

\newcommand{\TLeafNestedTwoLeaves}{%
\rtpic{
  \coordinate (r) at (0,0);
  \coordinate (a) at (-1.6,2);
  \coordinate (b) at (1.6,2);
  \coordinate (c) at (0.7,4);
  \coordinate (d) at (2.5,4);
  \draw[thin] (a)--(r)--(b);
  \draw[thin] (c)--(b)--(d);
}{r,a,b,c,d}}

\newcommand{\TLeafChainThree}{%
\rtpic{
  \coordinate (r) at (0,0);
  \coordinate (a) at (-1.6,2);
  \coordinate (b) at (1.6,2);
  \coordinate (c) at (1.6,4);
  \coordinate (d) at (1.6,6);
  \draw[thin] (a)--(r)--(b)--(c)--(d);
}{r,a,b,c,d}}

\newcommand{\TChainTwoLeafLeaf}{%
\rtpic{
  \coordinate (r) at (0,0);
  \coordinate (a) at (-2.2,2);
  \coordinate (b) at (-2.2,4);
  \coordinate (c) at (0,2);
  \coordinate (d) at (2.2,2);
  \draw[thin] (b)--(a)--(r)--(c);
  \draw[thin] (r)--(d);
}{r,a,b,c,d}}

\newcommand{\TChainTwoChainTwo}{%
\rtpic{
  \coordinate (r) at (0,0);
  \coordinate (a) at (-1.5,2);
  \coordinate (b) at (-1.5,4);
  \coordinate (c) at (1.5,2);
  \coordinate (d) at (1.5,4);
  \draw[thin] (b)--(a)--(r)--(c)--(d);
}{r,a,b,c,d}}

\newcommand{\TNestedTwoLeavesLeaf}{%
\rtpic{
  \coordinate (r) at (0,0);
  \coordinate (a) at (-1.6,2);
  \coordinate (b) at (-2.5,4);
  \coordinate (c) at (-0.7,4);
  \coordinate (d) at (1.6,2);
  \draw[thin] (r)--(a);
  \draw[thin] (b)--(a)--(c);
  \draw[thin] (r)--(d);
}{r,a,b,c,d}}

\newcommand{\TNestedThreeLeaves}{%
\rtpic{
  \coordinate (r) at (0,0);
  \coordinate (a) at (0,2);
  \coordinate (b) at (-2,4);
  \coordinate (c) at (0,4);
  \coordinate (d) at (2,4);
  \draw[thin] (r)--(a);
  \draw[thin] (b)--(a)--(c);
  \draw[thin] (a)--(d);
}{r,a,b,c,d}}

\newcommand{\TNestedLeafChainTwo}{%
\rtpic{
  \coordinate (r) at (0,0);
  \coordinate (a) at (0,2);
  \coordinate (b) at (-1.5,4);
  \coordinate (c) at (1.5,4);
  \coordinate (d) at (1.5,6);
  \draw[thin] (r)--(a);
  \draw[thin] (b)--(a)--(c)--(d);
}{r,a,b,c,d}}

\newcommand{\TChainThreeLeaf}{%
\rtpic{
  \coordinate (r) at (0,0);
  \coordinate (a) at (-1.6,2);
  \coordinate (b) at (-1.6,4);
  \coordinate (c) at (-1.6,6);
  \coordinate (d) at (1.6,2);
  \draw[thin] (c)--(b)--(a)--(r)--(d);
}{r,a,b,c,d}}

\newcommand{\TNestedChainTwoLeaf}{%
\rtpic{
  \coordinate (r) at (0,0);
  \coordinate (a) at (0,2);
  \coordinate (b) at (-1.5,4);
  \coordinate (c) at (-1.5,6);
  \coordinate (d) at (1.5,4);
  \draw[thin] (r)--(a);
  \draw[thin] (c)--(b)--(a)--(d);
}{r,a,b,c,d}}

\newcommand{\TDoubleNestedTwoLeaves}{%
\rtpic{
  \coordinate (r) at (0,0);
  \coordinate (a) at (0,2);
  \coordinate (b) at (0,4);
  \coordinate (c) at (-1.5,6);
  \coordinate (d) at (1.5,6);
  \draw[thin] (r)--(a)--(b);
  \draw[thin] (c)--(b)--(d);
}{r,a,b,c,d}}

\newcommand{\TChainFive}{%
\rtpic{
  \coordinate (r) at (0,0);
  \coordinate (a) at (0,2);
  \coordinate (b) at (0,4);
  \coordinate (c) at (0,6);
  \coordinate (d) at (0,8);
  \draw[thin] (r)--(a)--(b)--(c)--(d);
}{r,a,b,c,d}}

\hypersetup{colorlinks=true,linkcolor=red!70!black,linktocpage=false,citebordercolor=blue!70!black,citecolor=blue!70!black,anchorcolor=blue!70!black}

\definecolor{bg_yellow}{HTML}{FCFCE3}
\definecolor{title_yellow}{HTML}{F4EEAC}

\DeclareFontEncoding{LS1}{}{}
\DeclareFontSubstitution{LS1}{stix}{m}{n}
\DeclareSymbolFont{stixletters}{LS1}{stix}{m}{it}
\DeclareMathAccent{\backvec}{\mathord}{stixletters}{"91}

\newtcolorbox{road2}{
  colback=bg_yellow,
  enhanced,
  title=,
  attach boxed title to top left,
  fonttitle=\bfseries,
  coltitle=black,
  toprule=0pt,
  bottomrule=0pt,
  rightrule=0pt,
  leftrule=3pt,
  arc=0mm,
  skin=enhancedlast jigsaw,
  sharp corners,
  colframe=bg_yellow,
  colbacktitle=title_yellow,
  boxed title style={
    frame code={
      \fill[title_yellow] 
        (frame.south west) -- 
        (frame.north west) -- 
        (frame.north east) -- 
        ([xshift=3mm]frame.east) -- 
        (frame.south east) -- 
        cycle;

      \draw[line width=1mm, title_yellow] 
        ([xshift=2mm]frame.north east) -- 
        ([xshift=5mm]frame.east) -- 
        ([xshift=2mm]frame.south east);

      \draw[line width=1mm, title_yellow] 
        ([xshift=5mm]frame.north east) -- 
        ([xshift=8mm]frame.east) -- 
        ([xshift=5mm]frame.south east);

    }
  }
}

\UseTblrLibrary{booktabs}
\theoremstyle{definition}
\newtheorem{definition}{Definition}[section]

\newtheorem{example}{Example}[section]

\theoremstyle{remark}
\newtheorem*{remark}{Remark}

\newtheorem{proposition}{Proposition}[section]
\newtheorem{theorem}{Theorem}[section]

\title{Explicit and Effectively Symmetric Schemes for Neural SDEs on Lie Groups}

\author{%
  Daniil Shmelev\\
  Department of Mathematics\\
  Imperial College London\\
  \texttt{daniil.shmelev23@imperial.ac.uk} \\
  \And
  Luke Thompson\\
  University of Sydney\\
  \texttt{luke.thompson@sydney.edu.au} \\
  \And
  Cristopher Salvi\\
  Department of Mathematics\\
  Imperial College London\\
  \texttt{c.salvi@imperial.ac.uk} \\
}

\providecommand{\sshu}{\mathbin{\joinrel{\,\scriptscriptstyle\amalg\hskip -2.5pt\amalg}\,}}

\begin{document}

\maketitle

\begin{abstract}
  Backpropagation through (neural) SDE solvers is traditionally approached in two ways: discretise-then-optimise, which offers accurate gradients but incurs prohibitive memory costs; and optimise-then-discretise, which achieves constant memory cost by solving an auxiliary backward SDE, but suffers from slower evaluation and gradient approximation errors. Algebraically reversible solvers promise both memory efficiency and gradient accuracy, yet existing methods such as Reversible Heun are often unstable under complex models and large step sizes, and their non-standard auxiliary-state structure obstructs extension to manifold-valued SDEs. Building on the recently introduced \textit{Explicit and Effectively Symmetric (EES)} schemes -- a class of stable, near-reversible explicit Runge--Kutta methods -- we address both limitations of existing schemes. We extend EES schemes from ODEs to SDEs and show that they admit an efficient Williamson $2N$-storage realisation. Bazavov's commutator-free construction then lifts these schemes to arbitrary Lie groups and homogeneous spaces. To our knowledge, this is the first explicit \mbox{(near-)reversible} integrator in this setting, unlocking the reversible adjoint approach for manifold-valued problems. On Euclidean neural SDE benchmarks, our schemes improve stability under stiff drift and large steps compared with other reversible solvers, while the commutator-free lift reduces memory by up to an order of magnitude on manifold-valued problems versus other baselines. These results establish effectively symmetric integration as a unified, geometry-aware foundation for memory-efficient and stable training of neural SDEs.
\end{abstract}

\section{Introduction}
\emph{Neural stochastic differential equations (NSDEs)} have recently emerged as a flexible tool for modelling stochastic dynamics, with training typically cast as a distribution-matching problem between generated and observed trajectories. Several approaches have been proposed in the literature, differing mainly in the choice of discriminating divergence. SDE-GANs~\citep{kidger2021neural} use the 1-Wasserstein distance, while Latent SDEs~\citep{li2020scalable} optimise with respect to the KL divergence via variational inference. Another alternative proposed by~\cite{issa2023non} trains neural SDEs non-adversarially using maximum mean discrepancies (MMD) with \emph{signature kernels}~\citep{kiraly2019kernels, salvi2021signature, lemercier2024log}, a recently introduced family of efficient kernels on path space ~\citep{salvi2021higher, lemercier2021distribution, pannier2024path, cirone2025rough}. Across these formulations, training requires differentiating through an SDE solver, making the numerical method a central part of the learning algorithm.

A second pressure on the design of an SDE integrator comes from the geometry of the state space. Many stochastic dynamics of practical interest evolve in non-Euclidean spaces. In biology, torus-valued processes underpin peptide design and torsion-angle prediction~\citep{linPPFlowTargetawarePeptide2024, zhuTrivializedMomentumFacilitates2025, costa2024equijumpproteindynamicssimulation}. In motion capture, articulated poses evolve on products of rotations~\citep{https://doi.org/10.24432/c57g8x, zeng2023latent, bastian_continuous-time_2025}. In finance, asset-return covariances evolve on the SPD manifold~\citep{noureldinMultivariateHighfrequencybasedVolatility2012, johansson_simple_2023}. These geometries share a common structure: they are Lie groups, products of Lie groups, or are generated by Lie-group actions. In such settings, leaving the manifold breaks the required geometric constraints, and training therefore requires differentiating through a geometry-preserving solver.

\begin{wrapfigure}{r}{0.5\textwidth}
    \vspace{-1em}
    \centering
    \includegraphics[trim={1mm 1mm 1mm 1mm},clip,width=\linewidth]{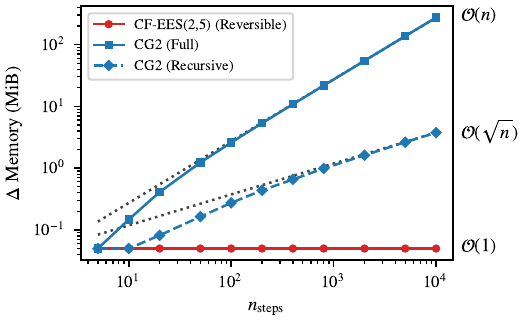}
    \caption{\textit{Growth in memory requirement for one forward and backward solve of a batch of 1024 SDEs on the 7-torus $\mathbb{T}^7$. Our method is in \textcolor[HTML]{D62728}{red}.}}
    \label{fig:torus_scaling}
\end{wrapfigure}
\textit{Adjoint methods} provide a family of approaches to perform this backpropagation through the solver. A first approach, known as \emph{discretise-then-optimise}, directly backpropagates through the solver's internal operations. This yields the exact gradient of the discretised computation and is computationally efficient, but requires storing all intermediate states, making it memory-intensive ($\mathcal{O}(n)$); we refer to this method as the \textbf{Full} adjoint. A second approach, \emph{optimise-then-discretise}, instead derives a backward-in-time adjoint equation and solves it numerically using another call to the solver. This avoids storing the full forward trajectory, resulting in constant memory with respect to the number of solver steps. However, it generally does not return the exact gradient of the discretised solver and is often slower due to the need to recompute forward trajectories during the backward pass. Practical implementations use recursive checkpointing~\citep{stumm2010new} to balance memory and recomputation, yielding an intermediate $\mathcal{O}(\sqrt{n})$ memory regime; we refer to this hybrid as the \textbf{Recursive} adjoint.

A third option leverages \emph{algebraically reversible solvers}, which enable exact reconstruction of the solution trajectory from the terminal state, and hence permit accurate and memory-efficient backpropagation in $\mathcal{O}(1)$ memory; we refer to this as the \textbf{Reversible} adjoint. In the setting of an autonomous ODE
\begin{equation} \label{eq:ode}
    \mathrm{d}y_t = f(y_t) \, \mathrm{d}t,
\end{equation}

a one step method $y_{n+1} = y_n + \Phi_h(y_n)$ is said to be \emph{reversible} or \emph{symmetric} if a step of the method starting from $y_1$ with a negative step size exactly recovers the initial condition $y_0$, that is, $\Phi_{-h} = \Phi^{-1}_h$. While reversible schemes offer an efficient approach to backpropagation through differential equations, such schemes are difficult to construct. It is well known that Runge--Kutta schemes are reversible only if they are implicit, making them unsuitable for applications to Neural ODEs. More generally, symmetric parasitism-free general linear methods cannot be explicit \citep{butcher2016symmetric}.

To overcome this problem, existing reversible methods proposed in literature, such as the Reversible Heun method \citep{kidger2021efficient} and the McCallum-Foster methods \citep{mccallum2024efficient}, track auxiliary states as part of the integration. While this allows explicit reversible schemes, it comes at the cost of stability, with both methods known to be unstable and prone to failure when integrating complex equations \cite{zhang2021path}. In addition to these stability concerns, the non-standard constructions of these schemes make it unclear how to generalise them to non-Euclidean spaces. To our knowledge, there are no efficient reversible numerical schemes on Lie groups, making such $\mathcal{O}(1)$ memory efficiency unattainable for manifold-valued NSDEs.

A potential solution to the difficulties of reversible solvers comes with the class of \emph{Explicit and Effectively Symmetric (EES)} Runge--Kutta schemes~\citep{shmelev2025explicit}. EES schemes relax exact reversibility to reversibility within a controlled tolerance, giving stable explicit methods that are empirically indistinguishable from truly symmetric schemes while retaining stability comparable to classical schemes such as RK3 and RK4. Importantly, unlike auxiliary-state reversible solvers, EES schemes retain the structure of ordinary Runge--Kutta methods making them more suited to geometric integration. However, the existing formulation in \citet{shmelev2025explicit} is limited to Euclidean ODEs: it does not treat stochastic dynamics, establish low-storage $2N$ realisations, or provide a construction capable of geometric integration over Lie groups.

Building on this prior work, we extend EES schemes to SDEs and derive Williamson $2N$ realisations, halving the constant-factor memory footprint of the resulting reversible solvers. We then use this $2N$ structure to construct a commutator-free lift to Lie groups. The resulting schemes, denoted $\mathrm{EES}_{\mathcal R}$ in Euclidean space and $\mathrm{CF\text{-}EES}$ on Lie groups, improve the stability of reversible Euclidean NSDE training and, to our knowledge, provide the first explicit near-reversible route to constant-memory training of manifold-valued NSDEs. Concretely, we make three main contributions:
\begin{enumerate}[]
    \item \textbf{Stable, reversible SDE integration.} We extend EES schemes to SDEs, obtaining explicit effectively symmetric integrators applicable to neural SDE training, with substantially better stability than existing explicit reversible solvers.
    \item \textbf{Geometric reversible integration.} We show that EES schemes admit a low-memory $2N$ formulation which naturally lifts to a commutator-free Lie group integrator family, $\mathrm{CF\text{-}EES}$. We then rigorously prove the convergence and order of these new SDE schemes.
    \item \textbf{Practical training benefit.} Practically, the improved stability of $\mathrm{EES}$ schemes enhances the accuracy of Euclidean NSDEs; while $\mathrm{CF\text{-}EES}$ cuts the memory requirements of Lie-group-valued NSDEs by an order of magnitude, or more.
\end{enumerate}
We release open-source implementations of the standard, low-memory, and Lie group EES schemes in \href{https://anonymous.4open.science/r/jax_geo_int/README.md}{JAX}, and \href{https://anonymous.4open.science/r/EffectivelySymmetric-jl/}{Julia} to facilitate adoption by the broader research community.

The paper is organized as follows. Section~\ref{sec:preliminaries} reviews reversible ODE solvers and their stability. Section~\ref{sec:ees} develops EES schemes for SDEs, including low-storage, adaptive, and manifold-valued variants, with convergence, stability, and backpropagation results. We compare our EES schemes against other reversible methods in Section~\ref{sec:experiments}, before concluding in Section~\ref{sec:conclusion}.

\section{Preliminaries}
\label{sec:preliminaries}
\paragraph{Existing reversible ODE and SDE solvers.}
The major drawback of classical reversible schemes is their low efficiency. It is well known that Runge--Kutta schemes are reversible only if they are implicit. More generally, symmetric parasitism-free general linear methods cannot be explicit ~\citep{butcher2016symmetric}. A limited number of efficient reversible solvers have been proposed in the literature. The asynchronous leapfrog integrator (ALF)~\citep{zhuang2021mali} for Neural ODEs overcomes the barrier of implicit schemes by tracking an additional state $v$ as part of the integration. Reversible Heun takes a similar approach for SDEs of the form $dy_t = g(t,y_t)dt + f(t,y_t)dW_t$. Although efficient, requiring only one evaluation of the drift $g$ and the diffusion $f$ per step, Reversible Heun is known to be inherently unstable.

\begin{theorem}{\cite[Theorem D.19]{kidger2021efficient}}\label{thm:rev_heun_stability}
    Suppose that the Reversible Heun method is used to obtain a solution $\{y_n, v_n\}_{n \geq 0}$ to the linear test ODE $dy = \lambda y\, dt$, where $\lambda \in \mathbb{C}$ and $y_0 \neq 0$. Then $\{y_n, v_n\}_{n \geq 0}$ is bounded if and only if $\lambda h \in [-i, i]$.
\end{theorem}

As remarked in ~\cite{kidger2021efficient}, this domain is also the absolute stability region for the reversible asynchronous leapfrog integrator ~\citep{zhuang2021mali}. This instability has proven to be a significant bottleneck in certain practical applications~\citep{zhang2021path, mccallum2024efficient}. \citet{mccallum2024efficient} proposed a method to transform any ODE integration method $y_{n+1} = y_n + \Psi_h(t,y_n)$ into one which is reversible, by coupling the action of the integrator with both positive and negative step sizes. The method offers a way of constructing reversible schemes with larger stability domains than those of the ALF and Reversible Heun integrators \cite[Theorem 2.3]{mccallum2024efficient}. However, the resulting stability domain of the transformed method is typically much smaller than that of the underlying method $\Psi$, and depends additionally on the coupling parameter. The McCallum-Foster methods were subsequently extended by~\citep{blasingameRexReversibleSolvers2025} to \emph{REX solvers} -- a class of algebraically reversible exponential RK/SRK solvers specifically designed for diffusion-model inversion.

\paragraph{EES schemes for ODEs.}
EES schemes~\citep{shmelev2025explicit} are a class of explicit Runge--Kutta methods which offer an efficient approach to reversible integration without compromising on stability. Given positive integers $m \geq n$, an explicit Runge--Kutta scheme $\Phi_h$ is said to be an $\mathrm{EES}(n,m)$ scheme if $\Phi_h$ is of order $n$ and $\Phi_{-h} \circ \Phi_h$ recovers the initial condition of the ODE up to order $m$. When $m$ is large, such schemes exhibit near-reversible behaviour, which is often sufficient in practice. In~\cite{shmelev2025explicit}, Butcher tableaux are derived for 3-stage $\mathrm{EES}(2,5)$ and 4-stage $\mathrm{EES}(2,7)$ schemes. Although the emphasis there is mostly on $\mathrm{EES}(2,7)$, for our Neural SDE applications, we restrict attention to $\mathrm{EES}(2,5)$ as the additional accuracy of $\mathrm{EES}(2,7)$ does not justify the extra stage in the NSDE setting, as shown in Figure~\ref{fig:ees25_vs_27_grad_error}. Proposition~\ref{prop:EES_2_5} gives the general one-parameter family of $\mathrm{EES}(2,5;x)$ tableaux.

\pagebreak[3]
\begin{wrapfigure}[12]{r}{0.44\textwidth}
    \vspace{-1.2em}
    \centering
    \includegraphics[clip,width=\linewidth]{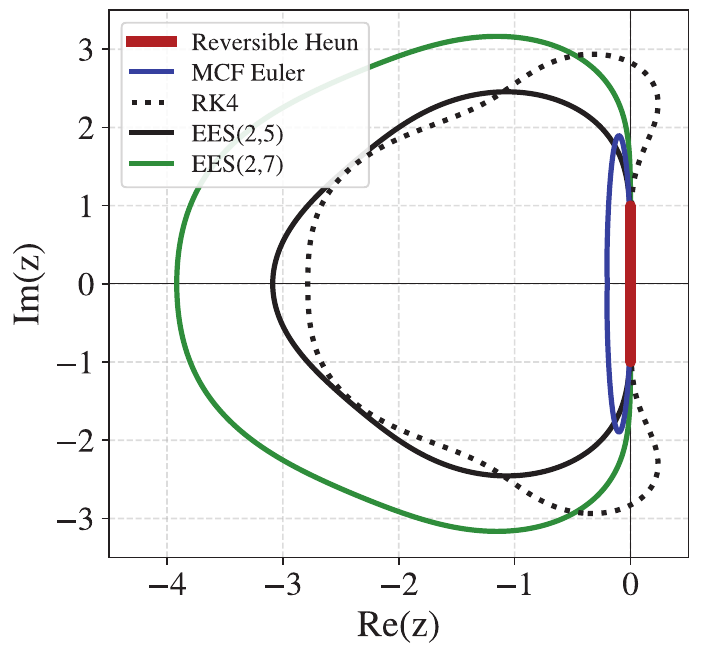}
    \captionsetup{oneside,margin={0.2cm,0cm}}
    \caption{\textit{Stability domains for $\mathrm{EES}(2,5)$ and $\mathrm{EES}(2,7)$ compared to RK4, MCF Euler, and Reversible Heun.}}
    \label{fig:ees25_stability}
    \vspace{-1em}
\end{wrapfigure}
\refstepcounter{proposition}\label{prop:EES_2_5}%
\noindent\textbf{Proposition \theproposition} (\cite[Proposition 8.4]{shmelev2025explicit})\textbf{.}\hspace{0.4em}\textit{For $x \in \mathbb{R} \setminus \{1,\pm \tfrac{1}{2}\}$, the 3-stage $\mathrm{EES}(2,5;x)$ Runge--Kutta scheme has Butcher tableau}\nopagebreak\par\nopagebreak\smallskip\nopagebreak
\noindent\hfill$\renewcommand{\arraystretch}{1.4}\setlength{\arraycolsep}{3pt}
\begin{array}{c|ccc}
0 &&&\\
\tfrac{1+2x}{4(1-x)} & \tfrac{1+2x}{4(1-x)} &&\\
\tfrac{3}{4(1-x)} & \tfrac{(4x-1)^2}{4(x-1)(1-4x^2)} & \tfrac{1-x}{1-4x^2} &\\
\hline
& x & \tfrac{1}{2} & \tfrac{1}{2} - x
\end{array}$\hfill\null
\par\smallskip

The stability region for $\mathrm{EES}(2,5)$ is comparable to that of classical methods such as Kutta's RK4, but significantly larger than those of Reversible Heun and the MCF methods. Theorem~\ref{thm:ees25_ode_stability} gives the exact form of this region for $\mathrm{EES}(2,5;x)$.

\begin{theorem}\label{thm:ees25_ode_stability}
    Suppose that, for $x \neq 1, \pm \frac{1}{2}$, $\mathrm{EES}(2,5;x)$ is used to obtain a solution $\{y_n\}_{n \geq 0}$ to the linear test equation $dy = \lambda y\, dt$, where $\lambda \in \mathbb{C}$ and $y_0 \neq 0$. Then $y_n \to 0$ as $n \to \infty$ if and only if
    \begin{equation*}
        \left\lvert 1 + \rho + \frac{1}{2} \rho^2 + \frac{1}{8} \rho^3 \right\rvert < 1, \qquad \rho = \lambda h.
    \end{equation*}
\end{theorem}
This follows by a direct computation of the stability function
$R(\rho)=1+\rho+\frac{1}{2}\rho^2+\frac{1}{8}\rho^3$,
which is independent of $x$. Following \citet[Section 8.1]{shmelev2025explicit}, we fix $x=1/10$ to minimise leading error, and refer to $\mathrm{EES}(2,5;1/10)$ as \textit{the} $\mathrm{EES}(2,5)$ scheme.

\section{Explicit and Effectively Symmetric Schemes for SDEs on Lie Groups}\label{sec:ees}

This section develops the main technical contributions of the paper. We extend the Euclidean EES schemes of \citet{shmelev2025explicit} from the ODE setting to SDEs, establish their mean-square stability, show that all $\mathrm{EES}(2,5;x)$ and $\mathrm{EES}(2,7;x)$ schemes admit an efficient Williamson $2N$-storage form, and lift them via Bazavov's commutator-free construction to a new family of effectively symmetric $\mathrm{CF\text{-}EES}$ integrators on arbitrary homogeneous spaces. To our knowledge, $\mathrm{CF\text{-}EES}$ is the first explicit (near-)reversible integrator in this setting, unlocking the reversible-adjoint backpropagation pipeline for manifold-valued neural SDEs.

\paragraph{EES schemes for SDEs.} We apply EES schemes to SDEs in the canonical way. Given an SDE $\mathrm{d}y_t = f(y_t)dt + g(y_t) \circ\mathrm{d}W_t$, we rewrite the equation as being driven by $X_t\coloneqq(t, W_t)$ with vector field $(f,g)$ and apply the Runge--Kutta scheme with increments $\mathrm{d}X_t$ in place of $h$. For Brownian drivers, this gives the usual strong order $1/2$, and weak order $1$. More generally, the same construction applies to differential equations driven by rough paths (RDEs) following \cite{redmann2020runge}, with the
convergence rate governed by the driver regularity. For a detailed description of EES schemes applied to RDEs, we refer the reader to Appendix~\ref{appendix:RDE}. Convergence rates for EES schemes applied to RDEs (of which SDEs are a special case) are given in Appendix~\ref{appendix:convergence_simplified}, followed by convergence experiments in Appendix~\ref{app:convergence_experiments}.


\paragraph{Stability.}

As discussed in the introduction, $\mathrm{EES}$ schemes offer a stable alternative to reversible integration. While the stability of $\mathrm{EES}$ in the case of ODEs has been studied in \citet{shmelev2025explicit} and Section~\ref{sec:preliminaries}, we are interested in the stability of $\mathrm{EES}$ when applied to stochastic drivers for our applications to Neural SDEs. To evaluate the stability in the context of SDEs, we consider the \textit{mean-square stability}, which is widely used for analysis of stochastic integration methods in the literature ~\citep{higham2000mean, drummond1991computer, hernandez1993convergence, komori1995stahle, komori1994some, petersen1998general, saito1993t, saito1996stability, schurz1996asymptotical}. Given the test equations $\mathrm{d}y_t = \lambda y_t \mathrm{d}t + \mu y_t \mathrm{d}W_t$, where $\lambda, \mu \in \mathbb{C}$ and $y_0 \neq 0$ almost surely, a solution $\{y_n\}_{n \geq 0}$ derived from a numerical integrator is said to be \textit{mean-square stable} if $\lim_{n \to \infty} \mathbb{E}(|y_n|^2) = 0$. It follows in a similar fashion to Theorem \ref{thm:ees25_ode_stability} that $\mathrm{EES}(2,5;x)$ applied to this test equation is mean-square stable if and only if
\begin{equation*}
    \mathbb{E} \left[ \left\lvert 1 + \rho + \frac{1}{2} \rho^2 + \frac{1}{8} \rho^3 \right\rvert^2 \right] < 1,
\end{equation*}
where $\rho = \lambda \mathrm{d}t + \mu \mathrm{d}W_t \sim N(\lambda \mathrm{d}t, \mu^2 \mathrm{d}t)$. Figure
\ref{fig:ees_stoch_stability} shows 4 cross-sections of the stability domain, compared to those of RK3 and RK4. Along most cross-sections, $\mathrm{EES}(2,5)$ achieves similar
or greater stability than RK3 and RK4.

\begin{figure}[ht]
\centering
    \includegraphics[width = \textwidth,trim={0.95cm .3cm 0cm .4cm},clip]{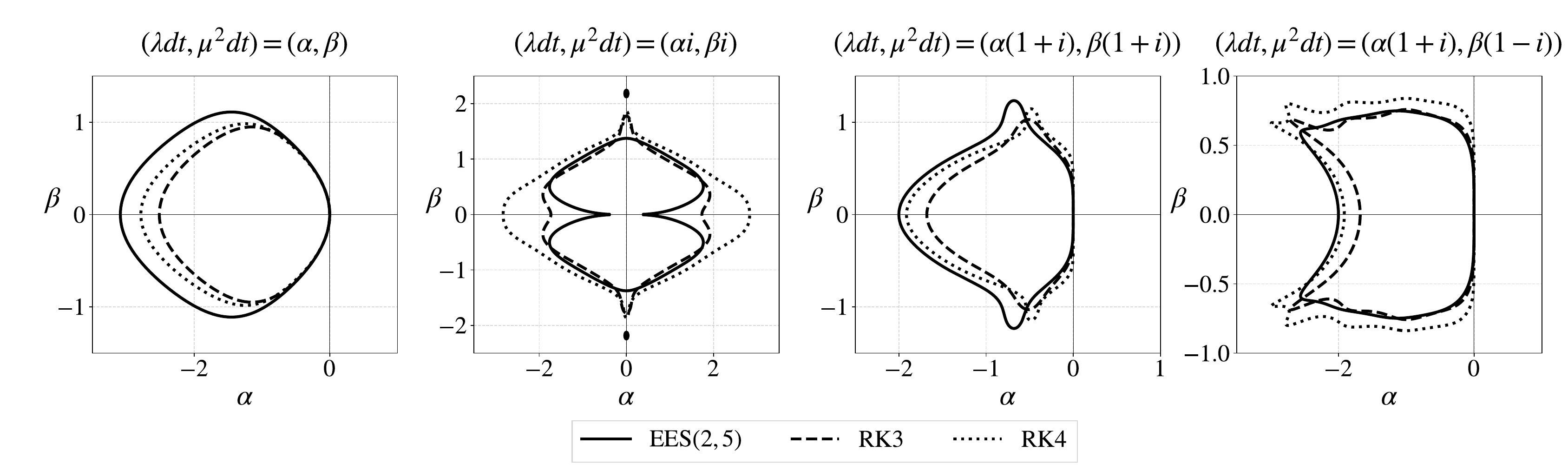}
    \caption{\textit{Cross sections of the mean-square stability domains of $\mathrm{EES}(2,5)$, RK3 and RK4.}}
    \label{fig:ees_stoch_stability}
\end{figure}

\paragraph{A $2N$ realization of EES Schemes.}
\label{sec:2n-ees}

The memory footprint of a classical Runge--Kutta step is proportional to the number of stages and thus becomes costly when the problem dimensionality is large. In \cite{williamson_low-storage_1980}, Williamson proposed a memory-efficient restructuring of the operations of a Runge--Kutta scheme into the following form
\begin{equation}
\label{eq:2n_rk}
\begin{aligned}
    \Delta Y_i &= A_i \Delta Y_{i-1} + h\, f(Y_{i-1}), \\
    Y_i &= Y_{i-1} + B_i \Delta Y_i,
\end{aligned}
\qquad i=1,\dots,s.
\end{equation}
Thus, for a state of size $N$, an $s$-stage implementation requires only $2N$ registers, compared with $(s+1)N$ for a standard explicit Runge--Kutta method. The conditions under which a Runge--Kutta scheme admits a Williamson 2N representation were formulated by \citet{bazavov_2n-storage_2025}.

\begin{theorem}[{\cite[Theorem~2]{bazavov_2n-storage_2025}}]
\label{thm:2N_relationship}
An explicit Runge--Kutta method admits a Williamson $2N$ representation if and only if
\begin{equation}\label{eq:williamson_2n_cond}
a_{ij}(b_{j-1}-a_{j,j-1})=(a_{i,j-1}-a_{j,j-1})b_j,
\qquad i=3,\dots,s,\quad j=2,\dots,i-1.
\end{equation}
\end{theorem}

In particular, a direct check of the conditions \eqref{eq:williamson_2n_cond} for $\mathrm{EES}(2,5;x)$ and $\mathrm{EES}(2,7;x)$ shows that both of these classes admit the Williamson 2N form.

\begin{proposition}
\label{thm:all-ees25-are-2n}
$\mathrm{EES}(2,5;x)$ and $\mathrm{EES}(2,7;x)$ are Williamson 2N for any admissible parameter $x$.
\end{proposition}
This formulation reduces the memory requirements of $\mathrm{EES}(2,5)$ and $\mathrm{EES}(2,7)$ from $4N$ and $5N$, respectively, to $2N$. As we will see in the following sections, the Williamson 2N formulation not only significantly reduces the memory footprint of the scheme but also allows lifting the scheme to commutator-free methods over Lie groups, or, more generally, homogeneous spaces.

\paragraph{Homogeneous spaces.}
A \emph{homogeneous space} is a smooth manifold $\mathcal{M}$ together
with a transitive action $\Lambda \colon G \times \mathcal{M} \to
\mathcal{M}$ of a Lie group $G$. Equivalently $\mathcal{M} \cong G/H$
for the isotropy subgroup $H \leq G$ at any chosen base point. We
write $\mathfrak{g} = T_e G$ for the Lie algebra of $G$. A vector field
$F$ on $\mathcal{M}$ is represented through a state-dependent generator
$\xi \colon \mathcal{M} \to \mathfrak{g}$ via the \emph{fundamental
vector field}
\[
\xi(y)_{\mathcal{M}}(y) := \frac{\mathrm{d}}{\mathrm{d}t}\bigg|_{t=0}
\Lambda\bigl(\exp(t\,\xi(y)),\, y\bigr),
\qquad F(y) = \xi(y)_{\mathcal{M}}(y),
\]
the infinitesimal version of the group action
(see Appendix~\ref{subsec:homogeneous_space_integration}).

\paragraph{Commutator-free methods.}
A natural strategy to integrate $F$ on $\mathcal{M}$ is the Munthe-Kaas (RKMK) approach~\citep{munthe-kaasRungeKuttaMethodsLie1998}, which uses the exponential map to pull each step back to an equation on the Lie algebra and integrates it there with a classical Runge--Kutta scheme. However, achieving an order higher than two requires evaluating nested commutators of the stage generators, which is computationally expensive. \emph{Crouch--Grossman (CG) methods}~\citep{crouch_numerical_1993} avoid this by composing single-slope exponentials of the form $\exp(h\,\alpha_{ij} K_j)$, with $K_j = \xi(Y_{j-1}) \in \mathfrak{g}$, so that no Lie brackets appear. \emph{Commutator-free (CF) methods}~\citep{celledoni_commutator-free_2003} generalise this by allowing each exponential's argument to be a real linear combination $\sum_j a_{l;ij} K_j$ of stage generators. More efficient CF variants further minimise the number of exponentials and storage required per step by reusing exponentials across stages~\citep{bazavov_commutator-free_2020}. Full descriptions of RKMK, CG, and CF methods are given in Appendices~\ref{subsec:rkmk_integrators}--\ref{subsec:cf_integrators}.

\paragraph{Bazavov's $2N$ commutator-free lift.}
\citet{bazavov_commutator-free_2020} showed that any explicit Williamson $2N$ Runge--Kutta scheme can be lifted to a commutator-free method on a homogeneous space. The Euclidean update $Y_l = Y_{l-1} + B_l \delta_l$ is replaced by the action $Y_l = \Lambda(\exp(B_l \delta_l), Y_{l-1})$, while the increment recurrence $\delta_l = A_l \delta_{l-1} + h K_l$ on the Lie algebra is unchanged. Concretely, given a Williamson $2N$ scheme with coefficients $(A_l, B_l)_{l=1}^s$ with $A_1 = 0$, the update rule from $y_n \in \mathcal{M}$ to $y_{n+1}$ is
\begin{equation}
\label{eq:bazavov-cf-lift}
\begin{aligned}
  Y_0 &:= y_n, \qquad \delta_0 := 0, \\
  K_l &= \xi(Y_{l-1}), \\
  \delta_l &= A_l\, \delta_{l-1} + h\, K_l, \\
  Y_l &= \Lambda\bigl(\exp(B_l\, \delta_l),\, Y_{l-1}\bigr),
  \qquad l = 1, \ldots, s,
\end{aligned}
\end{equation}
and $y_{n+1} \coloneqq Y_s$. Here $Y_l \in \mathcal{M}$ are the internal stage values and $\delta_l \in \mathfrak{g}$ is the current Lie-algebra increment; only these two quantities are stored at any time, preserving the two-register low-storage pattern. On a flat manifold ($\Lambda(\exp(v), y) = y + v$, $\xi(y) = f(y)$), the recurrence collapses to \eqref{eq:2n_rk}.

\paragraph{The $\mathrm{CF\text{-}EES}$ family.}
Since $\mathrm{EES}(2,5;x)$ and $\mathrm{EES}(2,7;x)$ are Williamson $2N$ for every admissible $x$ (Proposition~\ref{thm:all-ees25-are-2n}), Bazavov's lift~\eqref{eq:bazavov-cf-lift} applies and produces the $\mathrm{CF\text{-}EES}(2,5;x)$ and $\mathrm{CF\text{-}EES}(2,7;x)$ schemes on any homogeneous space $\mathcal{M}$. Substituting the $2N$ coefficients of Section~\ref{sec:2n-ees} into~\eqref{eq:bazavov-cf-lift} gives the explicit recurrences with $s = 3$ and $s = 4$ stages respectively; the explicit reused-stage form of $\mathrm{CF\text{-}EES}(2,5;x)$ is recorded in Appendix~\ref{subsec:cfees25-general-form} and memory-compute optimality is shown in \ref{app:lie_integrator_complexity}.

\begin{remark}
    We note that, among existing explicit reversible solvers, the above lift to a commutator-free method is unique to EES thanks to its Runge--Kutta form and Williamson 2N properties. As Reversible Heun and McCallum-Foster methods are not of Runge--Kutta form, they do not admit an analogous lift. A manifold extension of those schemes would require replacing affine state operations with non-canonical group operations.\footnote{Affine combinations are not well-defined on a general manifold. On a Lie group, logarithmic coordinates supply a substitute, but only after fixing a base point or trivialisation.}
\end{remark}

\paragraph{Order and reversibility of CF-EES on homogeneous spaces.}
The order and near-reversibility properties of the Euclidean EES schemes carry over to the manifold lift.

\begin{theorem}\label{thm:cfees-order-main}
For every admissible $x$ and every homogeneous space $\mathcal{M} = G/H$, $\mathrm{CF\text{-}EES}(2,5;x)$ applied to an ODE on $\mathcal{M}$ is of order $2$, and $\Phi_{-h} \circ \Phi_h$ recovers the initial condition up to order $5$. The analogous statement holds for $\mathrm{CF\text{-}EES}(2,7;x)$ with recovery up to order $7$.
\end{theorem}

The proof follows a symbolic computation of the Lie-Butcher series for $\mathrm{CF\text{-}EES}(2,5;x)$ and $\mathrm{CF\text{-}EES}(2,7;x)$ in terms of the free parameter $x$, which is compared to the Lie-Butcher series of the true solution to the ODE to extract the local order of the scheme. For the recovery of the initial condition, the LB series of $\Phi_{-h} \circ \Phi_h$ is computed similarly and compared to $0$. For further details, the reader is referred to Appendix~\ref{sec:cfees-order-conditions}.

\begin{remark}
The corresponding statement for the local order of the SDE schemes follows naturally via the formalism of rough path theory. The argument follows that of \cite{redmann2020runge} for Euclidean spaces, and is detailed in Appendix~\ref{sec:cfees-rde-extension}.
\end{remark}

\paragraph{Backpropagation through EES and CF-EES.}
Algorithm~\ref{alg:rk_backprop} of Appendix \ref{app:backprop_through_explicit_rk_methods} describes backpropagation through the Euclidean $\mathrm{EES}$ integrator for general RDEs. The commutator-free analogue (Algorithm~\ref{alg:homogeneous_cf_backprop}, Appendix \ref{app:backprop_through_homogeneous_cf_methods}) differs mainly in that the adjoint evolves on the cotangent bundle.

\section{Experiments}
\label{sec:experiments}

We evaluate the performance of $\mathrm{CF\text{-}EES}(2,5)$ on a range of problems covering Euclidean spaces, Lie groups and homogeneous spaces. Suitable reversible SDE baselines are limited in these settings. At the time of writing, Reversible Heun \cite{kidger2021efficient} is the only widely adopted explicit reversible solver. To broaden this comparison, we adapt the construction of \citet{mccallum2024efficient} for reversible ODE solvers to the SDE versions of the Euler and Explicit Midpoint methods. Algorithm~\ref{alg:rk_backprop}, together with the backpropagation procedure of \citet{mccallum2024efficient}, then enables efficient differentiation through the resulting schemes. The focus of our Euclidean experiments is the superior stability of EES schemes compared to these schemes.\par

Neither method admits a lift to Lie group integrators, and we are not aware of other explicit reversible schemes for Lie groups. Therefore, for problems on Lie groups or homogeneous spaces, we compare against non-reversible methods using memory-intensive full adjoints and emphasize the memory advantages of $\mathrm{CF\text{-}EES}$ over these schemes. We report metrics with two standard deviations, as well as runtimes and memory, in the Lie-group case. Appendix~\ref{app:additional_experiments} contains additional experiments on unstable/stiff dynamics, finance, and molecular dynamics.

\begin{wrapfigure}[15]{R}{0cm}
    \includegraphics[width = 0.40\textwidth,trim={4.5mm 4mm 4.1mm 0mm},clip]{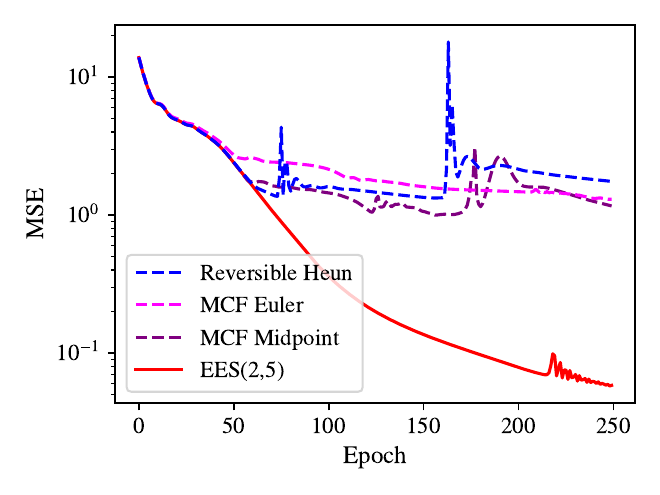}
    \captionsetup{oneside,margin={0.5cm,0cm}}
    \caption{\textit{Training MSE for OU dynamics with a fixed $f,g$ evaluation count.}}
    \label{fig:lsde}
\vspace{-0.5em}
\end{wrapfigure}
\paragraph{High volatility Ornstein–Uhlenbeck process.}
Consider learning the Ornstein–Uhlenbeck (OU) dynamics $\mathrm{d}y_t = \nu(\mu - y_t) \mathrm{d}t + \sigma \mathrm{d}W_t, y_0 \in \mathbb{R}$, under a high-volatility regime $\sigma \gg 0$. Specifically, we take $\nu=0.2, \mu=0.1$ and $\sigma=2$. Motivated by ~\citet{oh2024stable}, we take a Neural Langevin SDE (LSDE) defined by
\begin{equation*}
    \mathrm{d}z_t = g(z_t; \theta_g) \mathrm{d}t + f(t; \theta_f) \circ \mathrm{d}W_t,
\end{equation*}

with $z_0 = h(\mathbf{x}, \theta_h) \in \mathbb{R}^{d_z}$, where $h$ is a learnable affine function of the input data $\mathbf{x} = \{x_n\}_{n\geq 0}$, $x_n \in \mathbb{R}^2$, sampled from the true OU dynamics, and $g,f$ are neural networks parametrised by $\theta_g,\theta_f$ respectively. Architecture, training schedule, and loss are detailed in Appendix~\ref{app:ou_details}.\par\smallskip

Figure \ref{fig:lsde} shows the training loss using Reversible Heun, McCallum-Foster (MCF) methods, and $\mathrm{EES}(2,5)$, with the step size chosen such that the number of evaluations of $f,g$ is fixed between solvers. Such a choice yields comparable runtimes across all solvers, enabling a fair comparison. Table \ref{table:lsde} gives the number of evaluations of $f,g$ per step of the solvers, the chosen step size, the terminal MSE, and the total runtime of each solver. From Figure \ref{fig:lsde}, we see that for the initial $\sim 50$ epochs, the methods perform similarly. After this, $\mathrm{EES}(2,5)$ significantly outperforms the other methods, suggesting the model has begun to learn high-volatility dynamics which cause instability in the Reversible Heun and McCallum-Foster methods.\par\smallskip

\begin{table}[ht]
\centering
\caption{Metrics for OU dynamics. The step size is chosen such that the total number of evaluations of $f,g$ per integration is fixed.}
\label{table:lsde}
\renewcommand{\arraystretch}{1.1}
\begin{tabular}{lcccc}
\toprule
Method & \# Eval. / Step & Step Size & Terminal MSE & Runtime (s) \\
\midrule
Reversible Heun & 1 & $1/12$ & 1.02 & 368.2 \\
MCF Euler & 2 & $1/6$ & 1.30 & 307.5 \\
MCF Midpoint & 4 & $1/3$ & 1.17 & 279.5 \\
$\mathrm{EES}(2,5)$ & 3 & $1/4$ & 0.05 & 261.3 \\
\bottomrule
\end{tabular}
\vspace{-0.5em}
\end{table}

\paragraph{Stochastic Volatility.}
\label{subsec:stochastic_volatility}

Following the empirical observations of \citet{gatheral_volatility_2014}, rough volatility models posit that log-volatility behaves as a rough fractional process, with Hurst parameter \(H \ll \tfrac{1}{2}\). In direct discretisations of the rough driver, reduced regularity sharply increases the timestep budget required to attain a target error \(\varepsilon\), from \(\varepsilon^{-2}\) in the Brownian case to \(\varepsilon^{-6.25}\) when \(H=0.33\) \cite{friz2014convergence}. While finite-dimensional Markovian lifts can mitigate this, they do so at the cost of a substantially enlarged state space. This makes rough volatility a natural stress test for reversible solvers: long trajectories are required, and the benefit of \(\mathcal{O}(1)\) adjoint memory grows with more timesteps.

We therefore train a neural SDE with the same architecture and hyperparameters as in \eqref{eq:neural_sde} on rough Bergomi dynamics using a signature kernel score-matching objective \cite{issa2023non}. We use a generous fixed total evaluation budget to ensure good path fidelity and to reveal the runtime advantages of the $2N$ recurrence at high evaluation counts. Table~\ref{table:rough_bergomi} shows that, in this regime, all methods attain comparable terminal MSE, while \(\mathrm{EES}(2,5)\) is the fastest by a clear margin. Thus, in the long-horizon setting, \(\mathrm{EES}(2,5)\) preserves the accuracy of the reversible baselines while achieving the most favourable runtime among the reversible methods considered. We present additional results in Appendix~\ref{app:further_stoch_vol} for Black-Scholes, classical Bergomi, a local stochastic volatility model, the Heston model, rough Heston model, and quadratic rough Heston model.

\begin{table}[t]
\centering
\caption{Metrics for rough Bergomi dynamics. The step size is chosen such that the total number of evaluations of $f,g$ per integration is fixed.}
\label{table:rough_bergomi}
\renewcommand{\arraystretch}{1.1}
\begin{tabular}{lcccc}
\toprule
Method & \# Eval. / Step & Step Size & Terminal MSE & Runtime (s) \\
\midrule
Reversible Heun & 1 & $1/504$ & $7.81 \scriptstyle{\pm 1.06}$ & 999.3 \\
MCF Euler & 2 & $1/252$ & $7.81 \scriptstyle{\pm 1.06}$ & 928.9 \\
MCF Midpoint & 4 & $1/126$ & $7.81 \scriptstyle{\pm 1.06}$ & 519.4 \\
$\mathrm{EES}(2,5)$ & 3 & $1/168$ & $7.81 \scriptstyle{\pm 1.06}$ & 405.4 \\
\bottomrule
\end{tabular}
\vspace{-0.5em}
\end{table}

\paragraph{Stochastic Kuramoto network on $T\mathbb{T}^{N}$.}
\label{subsec:experiment_torus_sde}

\begin{wrapfigure}[31]{r}{0.40\textwidth}
    \centering
    \begin{subfigure}[t]{0.40\textwidth}
        \centering
        \includegraphics[width=0.85\linewidth]{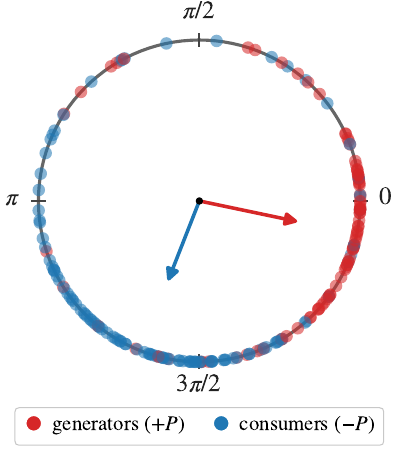}
        \caption{$200$ Kuramoto oscillators on $\mathbb{T}^{200}$ at partial synchronisation. At full synchronisation, the arrows would be superimposed.}
        \label{fig:kuramoto_trajectory}
    \end{subfigure}\\
    \vspace{1em}
    \begin{subfigure}[t]{0.40\textwidth}
        \centering
        \includegraphics[width=\linewidth]{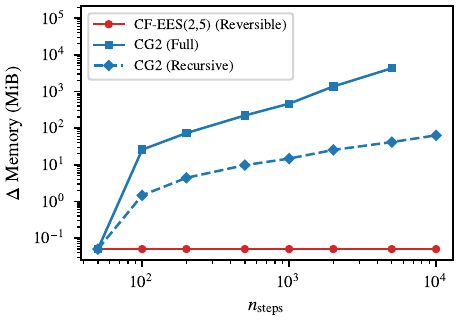}
        \caption{Memory complexity of CFEES and CG2 on the stochastic Kuramoto problem on $T\mathbb{T}^{1000}$ using different adjoints.}
        \label{fig:kuramoto_memory_scaling}
    \end{subfigure}
\end{wrapfigure}

The Kuramoto model \citep{acebron2005kuramoto} describes a network of $N$ coupled oscillators with phases $\theta_i \in S^{1}$ that interact through their pairwise differences. The model captures synchronisation phenomena across a wide range of physical and biological systems, such as circadian rhythms \cite{Lu2016}, neuronal oscillations \cite{Kitzbichler2009, Breakspear2010}, and the rotor-angle dynamics of synchronous machines on a power grid \citep{filatrella2008analysis}. Stochastic variants \citep{schmietendorf2014self,schafer2018dynamically} add the ability to model noise in the system. We use the second-order stochastic form of \citet[eq.~(1) with $K_2 = 0$]{olmi2024stochastic},
\begin{equation}\label{eq:kuramoto_2nd_order}
\begin{gathered}
m\,\ddot{\theta}_i \;=\; -\dot{\theta}_i \;+\; \Omega_i \;+\; \tfrac{K}{N}\textstyle\sum_j \sin(\theta_j - \theta_i) \;+\; \xi_i(t),\\
\langle \xi_i(t)\xi_j(s) \rangle = 2D\,\delta_{ij}\delta(t-s).
\end{gathered}
\end{equation}
with bimodal natural frequencies $\Omega_i \in \{+P, -P\}$ as in the power-grid generator/consumer split of \citet{filatrella2008analysis}.
Equation~\eqref{eq:kuramoto_2nd_order} evolves on the product Lie group $T\mathbb{T}^{N} \cong \mathbb{T}^{N}\times\mathbb{R}^{N}$. We train a neural SDE on $T\mathbb{T}^{N}$, with MLP drift and diffusion fields whose inputs are the periodic encoding $(\sin\theta, \cos\theta, \omega) \in \mathbb{R}^{3N}$ and whose outputs lie in the Lie algebra $\mathbb{R}^{2N}$. We train against synthetic trajectories of \eqref{eq:kuramoto_2nd_order} in the partial-synchronisation regime ($K=2$, $P=0.5$, $D=0.05$), using a multi-horizon wrapped energy score \citep{gneiting2007strictly}. Figure~\ref{fig:kuramoto_memory_scaling} shows the observed memory complexity of the training using $\mathrm{CF\text{-}EES}(2,5)$ with the reversible adjoint and CG2 with the recursive and full adjoints. Table~\ref{tab:kuramoto_quality} reports the test energy score of each method. $\mathrm{CF\text{-}EES}(2,5)$ attains a test energy score within roughly one standard deviation of the CG2 baselines whilst maintaining $\mathcal{O}(1)$ memory complexity. Architecture, training, and adjoint diagnostics are in Appendix~\ref{app:torus_details}.

\begin{table}[ht]
\centering
\caption{Test energy score on the stochastic Kuramoto problem. Step sizes are chosen so that all methods use the same number of vector-field evaluations per integration.}
\label{tab:kuramoto_quality}
\renewcommand{\arraystretch}{1.1}
\begin{tabular}{lccccc}
\toprule
Method & Adjoint & \#Eval. / Step & Step size & Test ES $\downarrow$ & Runtime (s) \\
\midrule
CG2 & Full & 2 & $1/75$ & $370.30 \scriptstyle{\pm 0.78}$ & $2{,}988$ \\
CG2 & Recursive & 2 & $1/75$ & $370.27 \scriptstyle{\pm 0.76}$ & $3{,}994$ \\
$\mathrm{CF\text{-}EES}(2,5)$ & Reversible & 3 & $1/50$ & $392.77 \scriptstyle{\pm 15.55}$ & $2{,}892$ \\
\bottomrule
\end{tabular}
\end{table}

\paragraph{Latent SDE on the sphere.}
\label{sec:sphere_latent_sde}

We reproduce the experiment of \citet{zeng2023latent} and train a variational latent SDE on the unit sphere $S^{n-1}$, viewed as the homogeneous space $\mathrm{SO}(n)/\mathrm{SO}(n-1)$, for human-activity classification on the UCI Human Activity benchmark \cite{reyesortiz2013human}. This dataset comprises 12-dimensional time series of wearable motion sensor readings, each labelled with one of seven activity classes at each time point. An mTAN attention encoder summarises the input sequence into\begin{wrapfigure}[15]{r}{0.42\textwidth}
    \centering
    \includegraphics[trim={1mm 1mm 1mm 1mm},clip,width=\linewidth]{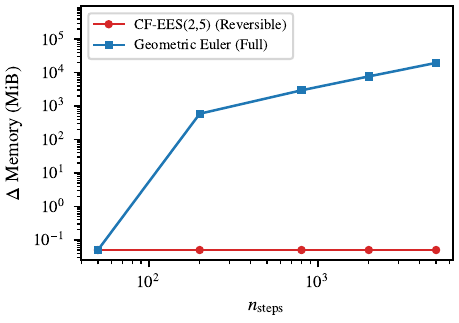}
    \caption{Memory complexity of CFEES and Geo E-M on the UCI Human Activity problem on $S^{15}$ using different adjoints.}
\label{fig:sphere_memory}
    \vspace{-3em}
\end{wrapfigure} a fixed-length context vector $h$ that conditions both the initial-state distribution $q(z_0\mid h)$ and the time-varying drift of a latent SDE on $S^{15}$. An MLP decoder maps each resulting latent state $z_t$ back to a Gaussian likelihood over the observed sensors, and a per-timepoint head predicts the activity class from $z_t$. \citet{zeng2023latent} use the geometric Euler-Maruyama integrator (Geo E-M), with the discretise-then-optimise (full) adjoint. We instead test the performance of $\mathrm{CF\text{-}EES}(2,5)$ with the reversible adjoint and also examine a stochastic RKMK (SRKMK) \cite{Muniz2023} form of ShARK \cite{foster2023high} with strong order 1 to see whether improved stochastic order aids learning.

Figure \ref{fig:sphere_memory} shows the observed peak memory requirement of one forward and backward pass of the model as a function of the number of steps taken by the latent SDE. As expected, the memory requirement of Geometric Euler-Maruyama grows linearly in the number of steps, while $\mathrm{CF\text{-}EES}(2,5)$ with the reversible adjoint retains $\mathcal{O}(1)$ memory complexity. At the same network-evaluation budget, SRKMK ShARK gives only a marginal accuracy improvement over $\mathrm{CF\text{-}EES}(2,5)$, despite its higher nominal stochastic order. This suggests that the lower strong order of $\mathrm{CF\text{-}EES}(2,5)$ is not a limiting factor in latent SDE learning.

\begin{table}[ht]
    \centering
    \caption{UCI Human Activity classification accuracy. The step size is chosen to keep the total number of NN evaluations per integration fixed.}
    \label{tab:sphere_parity}
    {\renewcommand{\arraystretch}{1.1}
    \begin{tabular}{lccccc}
        \toprule
        Method & Adjoint & \#Eval. / Step & Step Size & Test accuracy (\%) & Runtime (s) \\
        \midrule
        Geo E-M \citep{zeng2023latent} & Full & 1 & 1/30 & $86.25 \scriptstyle{\pm 0.76}$ & 366.3 \\
        $\mathrm{CG}2$ & Full  & 2 & 1/15 & $88.17 \scriptstyle{\pm 1.23}$ & 734.6 \\
        $\mathrm{CF\text{-}EES}(2,5)$ & Reversible  & 3 & 1/10 & $88.30 \scriptstyle{\pm 0.37}$ & 405.5 \\
        SRKMK ShARK & Full  & 3 & 1/10 & $88.66 \scriptstyle{\pm 0.59}$ & 657.2 \\
        \bottomrule
    \end{tabular}}
\end{table}

\section{Conclusions, limitations and future work}
\label{sec:conclusion}
In this paper, we have extended Explicit and Effectively Symmetric (EES) schemes \citep{shmelev2025explicit} to SDEs using the rough RK framework of \citet{redmann2020runge}. We have shown that the resulting schemes are stable via mean-square stability analysis, and that the resulting integrators admit clean generalisation to homogeneous spaces while retaining their $2N$ memory-optimality. These developments enable more stable training of Euclidean neural SDEs and unlock stable, constant-memory training of manifold neural SDEs, with applications in finance, biology, and robotics.
\vspace{-0.33em}
\paragraph{Limitations.}
Our constructions are fixed-step, and adaptive step sizing is nontrivial: step rejection requires restoring the previous state, which is incompatible with the two-register reversible implementation and would require a $3S^\ast$ reformulation \cite{KETCHESON20101763}. Additionally, on homogeneous spaces, local error estimates must be compared in a common linear space, e.g.\ via logarithmic maps or tangent-space trivialisations \cite[Chapter.~10]{Iserles_Munthe-Kaas_Nørsett_Zanna_2000}. A second structural limitation is that $\mathrm{CF\text{-}EES}$ relies intrinsically on a homogeneous-space representation with a transitive group action and tractable exponential map, and so does not apply to general Riemannian manifolds lacking such algebraic structure. Recent work \cite{bronascoHighOrderIntegration2025} suggests that such a generalisation is possible, albeit with different technical machinery.
\vspace{-0.33em}
\paragraph{Future directions.}
There are several potential extensions to this paper that are left for future research. Applications of EES schemes to more complicated models, including but not limited to Neural Jump SDEs \citep{jia2019neural, herrera2020neural}, Neural CDEs \citep{kidger2020neural} may be of interest, as would further examination of EES for accelerating Neural RDEs \citep{morrill2021neural}. An extension of EES schemes to include partitioned or adaptive step-size schemes would be valuable for training stiff neural differential equations. Further extension of our rough RK construction to higher levels of the signature suggests the possibility of a log-ODE-free neural RDE, escaping the primary cost of such methods \cite{redmannStateoftheArtNumericalSchemes2024}.

\paragraph{Reproducibility.} All experiments are available at \href{https://github.com/daniil-shmelev/EES-Neural-SDEs}{https://github.com/daniil-shmelev/EES-Neural-SDEs}, and the three released packages can be found at the links provided in Table~\ref{tab:new_packages}.

\begingroup
\sloppy
\bibliography{neurips_references}
\bibliographystyle{plainnat}
\endgroup
\clearpage

\appendix
\startcontents[appendices]
\renewcommand{\contentsname}{\textsc{Appendices}}
\printcontents[appendices]{}{1}{\setcounter{tocdepth}{2}}

\clearpage

\section{Algebraic Background}
\label{app:algebraic_background}
\subsection{The Connes-Kreimer Hopf Algebra}

We give a brief account of non-planar (labelled) rooted trees and the Connes-Kreimer Hopf algebra. We refer the reader to \cite{hoffman2003hopf} for a comprehensive presentation. A non-planar labelled rooted tree is defined as a graph $\tau = (V, E, r)$ with vertex set $V$, edge set $E$ and a root vertex $r \in V$, together with a set of vertex decorations drawn from $\{1, \ldots, d\}$. We denote the empty tree by $\emptyset$. Given trees $\tau_1, \ldots, \tau_m$, we write $[\tau_1, \ldots, \tau_m]_a$ to denote the tree formed by connecting the root vertices of $\tau_1, \ldots, \tau_m$ to a new root, which receives the label $a \in \{1, \ldots, d\}$. Non-planarity means tree order in $[\tau_1, \ldots, \tau_m]_a$ is irrelevant. Repeated trees will be denoted using power notation, for instance
\begin{equation*}
    [\tau_1,\tau_1, \tau_2, \tau_3, \tau_3, \tau_3]_a = [\tau_1^2, \tau_2, \tau_3^3]_a.
\end{equation*}

We write $|\tau|$ to denote the number of vertices in a tree. Additionally, we define the following combinatorial quantities, defined on unlabelled trees:
\begin{align*}
    \emptyset! &= 1, \quad \bullet! = 1, \quad [\tau_1, \ldots, \tau_m]! = |[\tau_1, \ldots, \tau_m]| \prod_{i=1}^m \tau_i!,\\
    \sigma(\emptyset) &= 1, \quad \sigma(\bullet) = 1, \quad \sigma([\tau_1^{k_1}, \ldots, \tau_m^{k_m}]) = \prod_{i=1}^m k_i! \sigma(\tau_i)^{k_i},\\
    \beta(\emptyset) &= 1, \quad \beta(\bullet) = 1, \quad \beta([\tau_1^{k_1}, \ldots, \tau_m^{k_m}]) = \binom{[\tau_1^{k_1}, \ldots, \tau_m^{k_m}]}{|\tau_1|, \ldots, |\tau_m|}\prod_{i=1}^m \frac{1}{k_i!} \beta(\tau_i)^{k_i}.
\end{align*}

We will refer to the commutative juxtaposition of trees as a forest. We write $\mathcal{T}$ to denote the set of all non-planar labelled rooted trees, and $\mathcal{T}_N \subset \mathcal{T}$ to denote the trees $\tau$ with $|\tau| \leq N$. The free commutative $\mathbb{R}$-algebra generated by $\mathcal{T}$ will be denoted $\mathcal{H}$. The Connes-Kreimer ~\citep{connes1999hopf} Hopf algebra on $\mathcal{H}$ is defined as follows. Multiplication $\mu : \mathcal{H} \otimes \mathcal{H} \to \mathcal{H}$ is defined as the commutative juxtaposition of two forests, extended linearly to $\mathcal{H}$. The multiplicative unit is defined to be the empty forest $\emptyset$. The counit map $\varepsilon: \mathcal{H} \to \mathbb{R}$ is defined by $\varepsilon(\emptyset) = 1$ and $\varepsilon(\tau) = 0$ for all non-empty trees $\tau \in \mathcal{H}$. The coproduct map is defined recursively by
\begin{equation*}
    \Delta(\emptyset) = \emptyset \otimes \emptyset, \quad  \Delta [\tau_1, \ldots, \tau_m]_a = [\tau_1, \ldots, \tau_m]_a \otimes \emptyset + (\mathrm{id} \otimes B^a_+)(\Delta \tau_1 \cdots \Delta \tau_m),
\end{equation*}
where $B^a_+(\tau_1 \cdots \tau_m) \coloneqq [\tau_1 \cdots \tau_m]_a$ for a forest $\tau_1 \cdots \tau_m$. The definition extends to a linear multiplicative map on $\mathcal{H}$. We will occasionally use Sweedler's coproduct notation $\Delta \tau = \sum_{(\tau)} \tau^{(1)} \otimes \tau^{(2)}$ and omit the definition of the antipode $S$ here, instead referring the reader to \citep{manchon2004hopf, hoffman2003hopf}. We denote the dual of the Connes-Kreimer Hopf algebra by $\mathcal{H}^*$. For $\varphi_1, \varphi_2 \in \mathcal{H}^*$, the convolution product is defined by $\varphi_1 * \varphi_2 = \mu_\mathbb{R} \circ (\varphi_1 \otimes \varphi_2) \circ \Delta$, with $\mu_\mathbb{R} \colon \mathbb{R} \otimes \mathbb{R} \to \mathbb{R}$ denoting multiplication in $\mathbb{R}$.

\subsection{B-Series Expansions of ODEs}\label{appendix:b_series_ode}

For any tree $\tau \in \mathcal{T}$, the so-called elementary differential $F(\tau)(y)$ ~\citep{butcher2016numerical} is defined recursively by
\begin{align*}
    &F(\emptyset)(y) = y, \quad F(\bullet_i)(y) = f_i(y),\\
    &F([\tau_1, \tau_2, \ldots, \tau_m]_i)(y) = f^{(m)}_i(y)(F(\tau_1)(y), F(\tau_2)(y), \ldots, F(\tau_m)(y)).
\end{align*}

Given a map $\varphi: \mathcal{T} \to \mathbb{R}$, the associated B-series is defined
\begin{equation*}
    B_h(\varphi, y_0) := \sum_{\tau \in \mathcal{T}} \frac{h^{|\tau|}}{\sigma(\tau)} \varphi(\tau) F(\tau)(y_0).
\end{equation*}

A key property of B-series is closure under composition \cite{hairer1974butcher,butcher2024b}: for two B-series, $B_h(\varphi_2, B_h(\varphi_1, y_0)) = B_h(\varphi_1 * \varphi_2, y_0)$, where $\varphi_1 * \varphi_2$ denotes the convolution product defined above. The exact solution of \eqref{eq:ode} has the B-series representation $y(h)=B_h(e,y_0)$, with $e(\tau)=1/\tau!$. Likewise, a Runge--Kutta method with coefficients $\{a_{ij}\}_{1\leq i,j\leq s}$ and $\{b_i\}_{1\leq i\leq s}$ has the B-series representation $B_h(\varphi,y_0)$, where \cite[Lemma~312B]{butcher2016numerical}
\begin{equation*}
    \varphi(\tau) := \sum_{i_1, \ldots, i_n} b_{i_1} \prod_{(k,\ell) \in E} a_{i_k, i_\ell}
\end{equation*}
for a tree $\tau = (V,E,r)$ with $|\tau| = n$. We refer the reader to ~\citep{hairer2006geometric,mclachlan2015butcher,butcher2021b} for a detailed account of B-series and the Butcher group.

\subsection{Branched Rough Paths}\label{appendix:branched_rp}

Let $\mathcal{H}$ be the Connes-Kreimer Hopf algebra of non-planar labelled rooted trees defined above.

\begin{definition}
    Let $\alpha \in (0,1]$. An $\alpha$-H\"older branched rough path is a map $\mathbf{X} : [0,T]^2 \to \mathcal{H}^*$ such that
    \begin{enumerate}
        \item for all $s,t \in [0,T]$ and $\tau_1, \tau_2 \in \mathcal{H}$,
        \begin{equation*}
            \langle \mathbf{X}_{s,t}, \tau_1 \rangle \langle \mathbf{X}_{s,t}, \tau_2 \rangle = \langle \mathbf{X}_{s,t}, \tau_1 \tau_2 \rangle,
        \end{equation*}
        \item for all $\tau \in \mathcal{H}$,
        \begin{equation*}
            \langle \mathbf{X}_{s,t}, \tau \rangle = \sum_{(\tau)} \langle \mathbf{X}_{s,u}, \tau^{(1)} \rangle \langle \mathbf{X}_{u, t}, \tau^{(2)} \rangle,
        \end{equation*}
        where $\Delta \tau = \sum_{(\tau)} \tau^{(1)} \otimes \tau^{(2)}$.
        \item for all $\tau \in \mathcal{H}$,
        \begin{equation*}
            \sup_{s \neq t} \frac{|\langle \mathbf{X}_{s,t}, \tau \rangle |}{|t - s| ^{\alpha |\tau|}} < \infty.
        \end{equation*}
    \end{enumerate}
\end{definition}

\begin{remark}
    As remarked in ~\citep{hairer2015geometric,gubinelli2010ramification}, the components $\langle \mathbf{X}_{s,t}, \tau \rangle$ with $|\tau| > N$ are determined by those with $|\tau| \leq N$, where $N$ is the largest integer such that $N \alpha \leq 1$.
\end{remark}

The space of $\alpha$-H\"older branched rough paths is a complete metric space under the metric

\begin{equation*}
    \varrho_\alpha(\mathbf{X}, \mathbf{Y}) := \sum_{\tau \in \mathcal{T}_N} \sup_{s \neq t} \frac{|\langle \mathbf{X}_{s,t} - \mathbf{Y}_{s,t}, \tau \rangle |}{|t - s|^{\alpha |\tau|}},
\end{equation*}

where $N = \lfloor 1 / \alpha \rfloor$.

\subsection{The Munthe-Kaas--Wright Hopf algebra}

Given a finite alphabet $A$, let $\mathcal{F}_A^{\mathrm{pl}}$ denote the set of $A$-decorated ordered planar rooted
forests, and write $\mathcal H_{\mathrm{MKW}}$ for the free $\mathbb R$-vector
space on $\mathcal{F}_A^{\mathrm{pl}}$. As before, for ordered trees
$\tau_1,\ldots,\tau_m$, $[\tau_1,\ldots,\tau_m]_a$ denotes the planar tree
formed by grafting $\tau_1,\ldots,\tau_m$ in the given order onto a new root
labelled $a$. The MKW Hopf
algebra structure on $\mathcal H_{\mathrm{MKW}}$~\citep{munthe2008hopf} is given by:
\begin{itemize}
\item \textbf{Product:} the shuffle product
  $\sshu : \mathcal H_{\mathrm{MKW}} \otimes \mathcal H_{\mathrm{MKW}} \to
  \mathcal H_{\mathrm{MKW}}$, summing all interleavings of two ordered
  forests preserving the relative order within each.
\item \textbf{Coproduct on trees:} the left-admissible-cuts coproduct
  \begin{equation*}
    \Delta_{\mathrm{MKW}}(\tau)
    = \tau \otimes \emptyset
    + \emptyset \otimes \tau
    + \sum_{c \in \mathcal C^{\ll}(\tau)} P^c(\tau) \otimes T^c(\tau),
  \end{equation*}
  where $\mathcal C^{\ll}(\tau)$ is the set of non-trivial cuts whose
  root-level edges form a \emph{left prefix} of the children at each vertex;
  $P^c(\tau)$ is the ordered forest of pruned subtrees, and $T^c(\tau)$ is
  the trunk. The coproduct is extended to forests via $\Delta_{\mathrm{MKW}}(\omega) = (\mathrm{id} \otimes B_-)
  \bigl(\Delta_{\mathrm{MKW}}(B_+(\omega)) - B_+(\omega) \otimes \mathbf 1\bigr)$,
  where $B_-$ extracts the children of a tree as an ordered forest.
\item \textbf{Counit:} $\varepsilon(\emptyset) = 1$ and $\varepsilon(\tau) = 0$
  for any non-empty $\tau$.
\end{itemize}

\section{RDE Framework and Convergence}\label{appendix:RDE}

We transform $\mathrm{EES}$ ODE schemes into RDE schemes using the framework of \citet{redmann2020runge}. Similarly to classical RK schemes for ODEs, the study of these methods is conducted through the formalism of B-series (see Appendix~\ref{app:algebraic_background}). A natural consequence of this analysis is that a general Runge--Kutta method for RDEs is given in terms of tree-iterated integrals of the underlying driving process. In practice, these tree-iterated integrals cannot be simulated directly as their distributions are often intractable. Following~\citep{deya2012milstein}, \cite{redmann2020runge} replaced these tree-iterated integrals with products of increments of the driving path. This substitution simplifies the derivation of Runge--Kutta coefficients and makes it feasible to establish order conditions up to any desired order.

\subsection{Simplified Runge--Kutta Methods for RDEs}

We consider rough differential equations (RDEs) of the form
\begin{equation}\label{eq:rde}
    dy_t = f(y_t)\, d\mathbf{X}_t,
\end{equation}
where $\mathbf{X}$ is an $\alpha$-H\"older branched rough path, $\alpha \in (0,1]$, and $f$ is smooth and bounded with bounded derivatives; see Appendix~\ref{appendix:branched_rp}. Following \citet{redmann2020runge}, we specialise to the geometric setting by assuming that there exist smooth approximations $\{X^h\}_{h>0}$ whose natural lifts $\{\mathbf{X}^h\}_{h>0}$ satisfy $\varrho_\alpha^g(\mathbf{X}^h,\mathbf{X})=\mathcal{O}(h^{r_0})$ for some $r_0>0$, where $\varrho_\alpha^g$ is the inhomogeneous geometric rough path metric. Such rates are known for Gaussian processes; see \citet{friz2014convergence}. Letting $y^h$ solve \eqref{eq:rde} driven by $\mathbf{X}^h$, a simplified Runge--Kutta scheme on an equidistant grid with stepsize $h$ is
\begin{equation}
    \begin{aligned}\label{eq:simple_rk}
        y^h_{n+1} &= y^h_n + \sum_{m=1}^d\sum_{i=1}^s b_i f_m(k_i) X^{(m)}_{t_n, t_{n+1}},\\
        k_i &= y^h_n + \sum_{m=1}^d\sum_{j=1}^s a_{ij} f_m(k_j) X^{(m)}_{t_n, t_{n+1}},
    \end{aligned}
\end{equation}
where $X^{(m)}_{t_n,t_{n+1}}$ denotes the $m$-th component increment of $X^h$ over $[t_n, t_{n+1}]$. Thus an ODE Runge--Kutta tableau induces an RDE scheme by weighting each tableau coefficient by the corresponding driver increment. Convergence rates for schemes of the form \eqref{eq:simple_rk} from \citet{redmann2020runge} are recalled in Appendix~\ref{appendix:convergence_simplified}.

\subsection{Order Conditions for General RDE Runge--Kutta Methods}

The simplified scheme \eqref{eq:simple_rk} is a special case of the general Runge--Kutta class considered by \citet{burrage1996high,burrage1998general,burrage2000order,redmann2020runge}. Namely, consider methods of the form
\begin{equation}
\begin{aligned}\label{eq:full_rk}
    y_{n+1} &= y_n + \sum_{m=1}^d \sum_{i=1}^s z_i^{(m)} f_m(k_i),\\
    k_i &= y_n + \sum_{m=1}^d \sum_{j=1}^s Z^{(m)}_{ij} f_m(k_j),
\end{aligned}
\end{equation}
where $Z^{(1)}, \ldots, Z^{(d)} \in \mathbb{R}^{s \times s}$ and $z^{(1)}, \ldots, z^{(d)} \in \mathbb{R}^s$. In particular, \eqref{eq:simple_rk} is recovered by taking $z_i^{(m)} = b_i X^{(m)}_{t_n,t_{n+1}}$ and $Z^{(m)}_{ij} = a_{ij}X^{(m)}_{t_n,t_{n+1}}$. We recall the local and global error rates for \eqref{eq:full_rk}, as formulated in \citet{redmann2020runge}, based on the adaptation of B-series to RDEs presented above.

\begin{definition}
    Given $h > 0$, define the maps $a, \varphi$ recursively over non-planar labelled rooted trees $\tau$ by setting $\varphi(\emptyset)(h) := (1,\ldots,1)^T \in \mathbb{R}^s$, where $\emptyset$ denotes the empty tree. For a tree $\tau = [\tau_1 \cdots \tau_n]_i$ formed by joining $\tau_1, \ldots, \tau_n$ by a new root labelled $i$, set
    \begin{align*}
        \varphi(\tau)(h) &:= \prod_{j=1}^n (Z^{(i)} \varphi(\tau_j)(h)),\\
        a(\tau)(h) &:= \left< z^{(i)}, \prod_{j=1}^n \varphi(\tau_j)(h)\right>.
    \end{align*}
\end{definition}

\begin{theorem}[{\citep{redmann2020runge}}]
    The general Runge--Kutta method given by \eqref{eq:full_rk} has a local error of order $(p+1)\alpha$ if and only if
    \begin{equation*}
        \left< \mathbf{X}_{t_0, t_0 + h}, \tau \right> = a(\tau)(h)
    \end{equation*}
    for all non-planar labelled rooted trees $\tau$ with $p$ or fewer nodes, i.e. $|\tau| \leq p$.
\end{theorem}

\begin{proposition}[{\cite[Proposition 4.1]{redmann2020runge}}]
    Let $y(t, y_0)$ denote the solution to \eqref{eq:rde} at time $t$ starting at $y_0$. Suppose the Runge--Kutta method \eqref{eq:full_rk} has a local error of order $(p+1)\alpha$, and there exists a constant $C_1 > 0$ such that
    \begin{equation*}
        |y(h, y_0) - y(h, \widetilde{y}_0)| \leq C_1 |y_0 - \widetilde{y}_0|,
    \end{equation*}
    for $h$ sufficiently small. Then there exists $C > 0$ such that
    \begin{equation*}
        \max_{n = 0, \ldots, N} |y(t_n) - y_n| \leq Ch^{(p+1)\alpha - 1}.
    \end{equation*}
\end{proposition}

\subsection{Convergence of simplified Runge--Kutta methods}
\label{appendix:convergence_simplified}

Given an ODE Runge-Kutta scheme $\Phi$, we will write $\mathcal{R}(\Phi)$ to denote the RDE scheme of the form in \eqref{eq:simple_rk} with the same coefficients $\{a_{ij}\}_{1 \leq i,j\leq s}$ and $\{b_i\}_{1 \leq i \leq s}$ as the ODE scheme.\par\medskip

\begin{theorem}[{\cite[Theorem 3.3]{redmann2020runge}}]
    Let $\Phi$ be an ODE Runge--Kutta scheme. The Runge-Kutta method $\mathcal{R}(\Phi)$ approximating $y^h$ has a local error of order $(p+1)\alpha$, i.e.
    \begin{equation*}
        y^h(t_0 + h) - y_1^h = \mathcal{O}(h^{(p+1)\alpha}),
    \end{equation*}
    if and only if the ODE Runge--Kutta method $\Phi$ is of order $p$.
\end{theorem}

\begin{theorem}[{\cite[Theorem 4.2]{redmann2020runge}}]\label{thm:ees_global_error}
    Let $\Phi$ be an ODE Runge--Kutta method of order $p$. Suppose that $f$ is $\text{Lip}^\gamma_b$ for some $\gamma > 1/\alpha$. Then $\mathcal{R}(\Phi)$ has a global error rate of $\eta = \min\{r_0, (p+1)\alpha - 1\}$, where $r_0$ is the convergence rate of the Wong-Zakai approximation. That is,
    \begin{equation*}
        \max_{n = 0, \ldots N} |y(t_n) - y^h_n| = \mathcal{O}(h^\eta).
    \end{equation*}
\end{theorem}

\subsection{Backpropagation through explicit Runge--Kutta methods}
\label{app:backprop_through_explicit_rk_methods}

The algorithm for backpropagation through an explicit Runge--Kutta scheme $\Phi$ of the form in \eqref{eq:simple_rk} is given in Algorithm \ref{alg:rk_backprop}. We assume the solver is applied to a (neural) RDE of the form
\begin{equation}
    dy^h_t = f(y^h_t; \theta) d\mathbf{X}^h_t,
\end{equation}
where $\theta$ are learnable parameters requiring backpropagation, trained with respect to a loss $L(\{y^h_n\}_{n=0}^N)$. As with all reversible schemes, a reverse step $\Phi^\mathrm{rev}$ is used to recover $y_n$ from $y_{n+1}$, followed by a backpropagation through the internal operations of the solver $\Phi$. The latter step is achieved by defining $z_i = f(k_i; \theta)$ and computing the derivatives $\partial L / \partial z_i$ and $\partial L / \partial k_i$ in reverse through the stages $i = s, s-1, \ldots, 1$. At each stage, a backpropagation algorithm is called to backpropagate the derivative $\partial L / \partial z_i$ through $f$, resulting in the derivative $\partial L / \partial k_i$ and a local derivative with respect to $\theta$, $d_\theta$.

\begin{algorithm}
\caption{Backpropagation through Explicit Runge--Kutta Schemes}\label{alg:rk_backprop}
\textbf{Input:} $y_{n+1}$, $\partial_{y_{n+1}} L$\\
\textbf{Input:} Running derivative with respect to $\theta$, $\partial_\theta L$\\
\textbf{Input:} Explicit RK method $\Phi$ of the form in \eqref{eq:simple_rk} with coefficients $\{a_{ij}\}_{1 \leq i,j \leq s}$ and $\{b_i\}_{1 \leq i \leq s}$.
\begin{algorithmic}

\State $y_n = \Phi^\mathrm{rev}(y_{n+1}, dX)$

\For {$i = s, \ldots 1$}
    \State $\partial_{z_i} L = b_i dX \cdot \partial_{y_{n+1}}L + \sum_{j=i+1}^s a_{ji} dX \cdot \partial_{k_j}L$
    \State $d_\theta, \partial_{k_i} L = \text{backprop}_f( \partial_{z_i}L)$
    \State $\partial_\theta L \mathrel{+}= d_\theta$
\EndFor

\State $\partial_{y_n} L = \partial_{y_{n+1}} L + \sum_{i=1}^s \partial_{k_i} L$\\
\Return $y_n, \partial_{y_n} L, \partial_\theta L$

\end{algorithmic}
\end{algorithm}

\section{Integrators on Homogeneous Spaces}

This section introduces geometric numerical integration over homogeneous spaces as an extension of the well-known Lie group case and describes an example based on the manifold of symmetric positive definite (SPD) matrices. We then provide some background on commutator-free methods, of which our $\mathrm{CF\text{-}EES}(2,5;x)$ belongs, and show its memory-compute optimality.

\subsection{Symmetric Integration over Homogeneous Spaces}
\label{subsec:homogeneous_space_integration}
For chart-based homogeneous space integrators, self-adjointness depends not only on the Butcher tableau of the underlying method, but also requires that the chosen local coordinates be compatible with time reversal. There are two natural ways to enforce this:
\begin{enumerate}[label=(\roman*)]
  \item use symmetric coordinates, centered at a geodesic or flow midpoint, in the style of self-adjoint Lie-group methods \cite{Zanna2001}, or
  \item use an embedded frozen-flow model whose frozen dynamics already integrate to an exactly
reversible map on \(M\).
\end{enumerate}

The first route is intrinsic, but generally leads to midpoint-centered constructions and hence away from the present
explicit frozen-flow framework. The second route is the one used here, which works for arbitrary homogeneous spaces.

Consider a homogeneous space $M$ endowed with a transitive action $\Lambda \colon G \times M \to M$
of a Lie group $G$. Rather than working with vector fields on $M$ directly, the key idea is to
represent them through elements of the Lie algebra $\mathfrak{g} = T_eG$: each $v \in \mathfrak{g}$
acts on $M$ via the \textbf{fundamental vector field}
\begin{equation}
    v_M(y) \coloneqq \frac{\mathrm{d}}{\mathrm{d}t}\bigg|_{t=0} \Lambda(\exp(tv),\, y),
\end{equation}
the infinitesimal version of the group action. A vector field $F$ on $M$ is then represented
through a state-dependent generator $\xi(y) \in \mathfrak{g}$ via $F(y) = \xi(y)_M(y)$,
replacing the role that Lie-algebra multiplication $vy$ plays on the group itself:
\begin{equation}
    gh \rightsquigarrow \Lambda(g,y), \qquad vy \rightsquigarrow v_M(y),
\end{equation}
for $g \in G$, $v \in \mathfrak{h}_y/\mathfrak{g}$, $y \in M$. This substitution lifts standard Lie-group
integration schemes from $G$ to $M$. In general, this construction is only unique up to the isotropy algebra $\mathfrak{h}_y = \{A \in \mathfrak{g} \colon A_m(y) = 0\}$, the set of generators whose induced vector field at $y$ is zero and thus produce no tangent motion. An example is given in Example~\ref{ex:isotropy}. To avoid this one fixes a representative of each tangent direction, typically by choosing a complement to the isotropy algebra and lifting only through that subspace.

\begin{road2}
\begin{example}[Example of degenerate choices due to the isotropy algebra]
\label{ex:isotropy}
Consider the sphere $S^2 \simeq \mathrm{SO}(3)/\mathrm{SO}(2)$ with the standard transitive action of $\mathrm{SO}(3)$ on $S^2 \subset \mathbb{R}^3$. At the north pole $p = e_3$, the isotropy subgroup is
\[
H_p = \{ R \in \mathrm{SO}(3) : Rp = p \} \simeq \mathrm{SO}(2),
\]
namely the rotations about the vertical axis. Its Lie algebra consists of the infinitesimal rotations that leave $p$ fixed to first order.

Equivalently, if $K \in \mathfrak{so}(3)$ generates a rotation about the $e_3$-axis, then its fundamental vector field vanishes at $p$:
\[
K_{S^2}(p)
=
\frac{\mathrm{d}}{\mathrm{d}t}\bigg|_{t=0} \exp(tK)p
=
0.
\]
Hence, if $A \in \mathfrak{so}(3)$ represents some tangent vector at $p$, then so does $A+K$ for any such $K$, since
\[
(A+K)_{S^2}(p) = A_{S^2}(p) + K_{S^2}(p) = A_{S^2}(p).
\]
Thus the lift of a tangent vector from $T_pS^2$ to $\mathfrak{so}(3)$ is not unique: one may add any element of the isotropy algebra without changing the induced tangent motion at $p$.
\end{example}
\end{road2}

The reverse of a frozen homogeneous-space flow is obtained by changing the sign of \(t\). Indeed, for a fixed generator \(v \in \mathfrak{g}\), by the group action identity \(\Lambda\) gives
\begin{equation}
    \Lambda\!\left(\exp(-tv), \Lambda(\exp(tv), y_0)\right)
    =
    \Lambda\!\left(\exp(-tv)\exp(tv), y_0\right)
    =
    \Lambda(e, y_0)
    =
    y_0.
\end{equation}
Hence the frozen flow is exactly reversible: its adjoint is obtained by time reversal, that is, by replacing \(t\) with \(-t\). This is a property of the frozen flow itself, before any recomputation of the generator along the trajectory.

\subsection{Runge--Kutta--Munthe-Kaas (RKMK) methods}\label{subsec:rkmk_integrators}
Runge--Kutta--Munthe-Kaas (RKMK) methods \cite{munthe-kaasRungeKuttaMethodsLie1998} extend classical explicit Runge--Kutta schemes to manifolds by transforming the original equation $\dot y = F(y)$ on $M$ into an equivalent equation on the linear space $\mathfrak{g}$. Writing the local solution as $y(t) = \Lambda(\exp(\sigma(t)), y_n)$ with $\sigma(0) = 0$, the curve $\sigma : I \to \mathfrak{g}$ satisfies
\begin{equation*}
    \dot\sigma = \mathrm{dexp}^{-1}_\sigma\bigl(f(\Lambda(\exp(\sigma), y_n))\bigr), \tag{pulled-back equation}
\end{equation*}
where $\mathrm{dexp}_\sigma : \mathfrak{g} \to \mathfrak{g}$ is the differential of the exponential map and admits the Bernoulli expansion
\begin{equation*}
    \mathrm{dexp}^{-1}_\sigma(v) = \sum_{k\geq 0} \frac{B_k}{k!}\, \mathrm{ad}_\sigma^k v, \qquad B_{2k+1} = 0\ \text{for}\ k\geq 1. \tag{Bernoulli expansion}
\end{equation*}
A classical Runge--Kutta tableau $(a_{ij}, b_i)$ is then applied directly to the pulled-back equation. The stage values $u_i = h\sum_j a_{ij} k_j \in \mathfrak{g}$ and slopes
\begin{equation*}
    k_i = \mathrm{dexp}^{-1}_{u_i}\bigl(f(\Lambda(\exp(u_i), y_n))\bigr) \in \mathfrak{g}, \qquad i = 1,\dots,s,
\end{equation*}
combine to give the manifold update $y_{n+1} = \Lambda\bigl(\exp(h\sum_i b_i k_i),\, y_n\bigr)$.

Achieving order $p$ requires truncating the Bernoulli expansion through $\mathrm{ad}_\sigma^k$ for $k \leq p - 2$ \cite{munthe-kaasRungeKuttaMethodsLie1998}, since each $\mathrm{ad}_\sigma^k v$ is a $k$-fold nested Lie bracket of stage generators. As a result, an order-$p$ RKMK method incurs a per-stage cost that grows with $p$ through nested commutator evaluations. Eliminating these nested commutators is the central motivation for the commutator-free constructions that follow.

\subsection{Crouch--Grossman (CG) methods}\label{subsec:cg_integrators}
Crouch--Grossman (CG) methods~\cite{crouch_numerical_1993} were the first explicit Runge--Kutta-style integrators on Lie groups to dispense with commutators entirely. Each stage and the final update are realized as compositions of exponentials, but each exponential's argument involves only a single stage slope. With left-trivialized stage slopes
\begin{equation*}
    K_i = f(Y_i) \in \mathfrak{g}, \qquad i = 1,\dots,s,
\end{equation*}
the stage values and update take the form
\begin{align*}
    Y_i &= \Lambda\!\left(
      \mathcal{T}\left\{ \prod_{j=1}^{i-1} \exp\bigl(h\, \alpha_{ij}\, K_j\bigr) \right\},\ y_n
    \right), \qquad i = 1,\dots,s, \\
    y_{n+1} &= \Lambda\!\left(
      \mathcal{T}\left\{ \prod_{i=1}^{s} \exp\bigl(h\, \beta_{i}\, K_i\bigr) \right\},\ y_n
    \right),
\end{align*}
with $\mathcal{T}$ denoting ordered multiplication in $G$. Because every exponential argument is a scalar multiple of a single stage slope, no Lie brackets appear, but the cost saving over RKMK comes at the price of additional exponential evaluations per step. CG methods are precisely the special case of the commutator-free framework of Section~\ref{subsec:cf_integrators} in which each linear combination $\sum_j a_{l;ij} K_j$ collapses to a single term.

\subsection{Commutator-Free Integrators}\label{subsec:cf_integrators}
Commutator free (CF) \cite{celledoni_commutator-free_2003} generalize Crouch-Grossman (CG) \cite{crouch_numerical_1993} methods by allowing multiple exponentials per stage, while still avoiding the explicit evaluation of commutators as in Runge--Kutta Munthe--Kaas methods \cite{munthe-kaasRungeKuttaMethodsLie1998}. For simplicity, we present the autonomous case; the non-autonomous case is handled by the standard augmentation in time. Given internal stage values $Y_i \in M$, we evaluate the left-trivialized stage slopes
\begin{equation*}
    K_i = f(Y_i) \in \mathfrak{g}, \qquad i=1,\dots,s. \tag{left-trivialized slope}
\end{equation*}
Each stage is then formed by acting on the current state $y_n$ with an ordered product of exponentials of linear combinations of previously computed slopes
\begin{equation*}
    Y_i
    =
    \Lambda\!\left(
        \mathcal{T}\left\{
            \prod_{l=1}^{L_i}
            \exp\left(
                h \sum_{j=1}^{i-1} a_{l;ij}\,K_j
            \right)
        \right\},
        y_n
    \right),
    \qquad i=1,\dots,s.
\end{equation*} 
where $Y_i$ is the stage value and $y_{n+1}$ is the next solution point. The step update then takes the form,
\begin{equation*}
    y_{n+1}
    =
    \Lambda\!\left(
        \mathcal{T}\left\{
            \prod_{l=1}^{L}
            \exp\left(
                h \sum_{i=1}^{s} \beta_{l;i}\,K_i
            \right)
        \right\},
        y_n
    \right),
\end{equation*}
where $L_i$ and $L$ denote the numbers of exponentials used in stage $i$ and in the final update, respectively, and $\mathcal{T}$ denotes ordered multiplication in $G$, with smaller indices appearing to the right.

Order conditions for the coefficients $\alpha_{l;ij}$ and $\beta_{l;i}$ are derived in~\cite{owren2006order} via a Lie--Butcher / planar tree calculus, the non-commutativity of group composition forcing ordered rather than unordered trees. The canonical order-$4$ CF construction of \citet{celledoni_commutator-free_2003} uses $5$ exponentials per step, with one stage value reused across two stage exponentials.

\subsection{Adjoint Sensitivities for Commutator-Free Neural ODEs}
The continuous adjoint system underlying $\mathrm{CF\text{-}EES}$ is obtained by applying \citet[Theorem~4.5]{wotte_geometric_2025} through the homogeneous-space lift introduced in Section~\ref{subsec:homogeneous_space_integration}.

\subsection{Theoretical Compute and Memory Requirements of Lie Group Integrators}
\label{app:lie_integrator_complexity}
The per-step cost of an \(s\)-stage Lie group integrator can be decomposed as
\[
C_{\mathrm{step}}
=
s\,C_{\mathrm{eval}} + N_{\exp}\,C_{\exp},
\]
where \(C_{\mathrm{eval}}\) is the cost of one vector-field evaluation and \(C_{\exp}\) is the cost of one group exponential. Since all methods considered below use \(s\) stage evaluations, the main distinction in compute is the number of exponentials per step. On the memory side, generic Runge--Kutta-style Lie group methods typically retain several intermediate stage quantities, whereas low-storage constructions aim to keep this requirement fixed. The key comparison is therefore not only the exponential count, but whether this count can be achieved without increasing the number of stage registers with \(s\).

For an \(s\)-stage Crouch--Grossman (CG) method, stage \(i\) uses \(i-1\) exponentials and the final update uses \(s\) more, giving
\[
N_{\exp}^{\mathrm{CG}}(s)
=
\sum_{i=1}^s (i-1) + s
=
\frac{s(s+1)}{2}.
\]
Thus, the exponential count of CG grows quadratically with the number of stages.

A general commutator-free (CF) method instead exponentiates linear combinations of previously computed stage slopes \citep{celledoni_commutator-free_2003}. Writing \(L_i\) for the number of exponentials used in stage \(i\) and \(L\) for the number used in the final update, its total exponential count is
\[
N_{\exp}^{\mathrm{CF}}(s)
=
\sum_{i=1}^s L_i + L.
\]
CG is recovered as the special case \(L_i=i-1\) and \(L=s\). In practice, CF methods are designed so that \(L_i\) and \(L\) remain small, reducing the exponential count from quadratic to linear in \(s\). For example, the 3-stage and 4-stage commutator-free schemes of \citet{celledoni_commutator-free_2003} require \(3\) and \(5\) exponentials per step, respectively.

A particularly efficient instance arises for 2N-storage methods. In the commutator-free formulation of \citet{bazavov_commutator-free_2020}, exponentials are reused across stages and only two quantities are stored: the current updated state and the current stage increment. In this representation, each stage uses exactly one exponential applied to an accumulated linear combination of previously computed slopes. Consequently, whenever a classical Runge--Kutta scheme admits 2N form, the induced commutator-free implementation uses
\[
N_{\exp}^{\mathrm{2N\text{-}CF}}(s)=s
\]
exponentials per step while retaining the two-register storage pattern. In this sense, 2N-CF simultaneously achieves linear exponential complexity and minimal storage within this Runge--Kutta-style commutator-free design space.

\begin{table}[ht]
\centering
\caption{Per-step group-exponential count and forward stage-storage for \(s\)-stage Lie group integrators. All methods use \(s\) vector-field evaluations so the comparison isolates exponential costs \(N_{\exp}C_{\exp}\). We denote the methods of \citet{celledoni_commutator-free_2003} by CMO CF and show their best-case scaling.}
\label{tab:exp_counts}
\begin{tabular}{lccc}
\toprule
Method & \(N_{\exp}\) per step & Scaling of \(N_{\exp}\) & Forward stage storage \\
\midrule
CG & \(\frac{s(s+1)}{2}\) & \(\mathcal{O}(s^2)\) & \(\mathcal{O}(s)\) \\
CMO CF & \(\sum_{i=1}^s L_i + L\) & \(\mathcal{O}(s)\) & \(\mathcal{O}(s)\) \\
2N-CF & \(s\) & \(\mathcal{O}(s)\) & \(2\) registers \\
\bottomrule
\end{tabular}
\end{table}

\section{New EES Solvers}
\label{app:new_ees_solvers}

We now record the general $2N$ representations of the schemes
$\mathrm{EES}(2,5;x)$ and $\mathrm{EES}(2,7;x)$ identified by
\citet{shmelev2025explicit}, valid for every admissible parameter $x$.

For $x \in \mathbb{R} \setminus \{1, \pm \tfrac{1}{2}\}$, the Williamson
$2N$ coefficients of $\mathrm{2N\text{-}EES}(2,5;x)$ are
\begin{align*}
B_1 &= \frac{2x+1}{4(1-x)}, & B_2 &= \frac{1-x}{1-4x^2}, & B_3 &= \frac{1-2x}{2}, \\
A_1 &= 0, & A_2 &= \frac{4x^2-2x+1}{2(x-1)}, & A_3 &= -\frac{4x^2-2x+1}{(2x-1)^2(2x+1)}.
\end{align*}
At $x = \tfrac{1}{10}$ these evaluate to
$(B_1, B_2, B_3) = (\tfrac{1}{3}, \tfrac{15}{16}, \tfrac{2}{5})$ and
$(A_2, A_3) = (-\tfrac{7}{15}, -\tfrac{35}{32})$.

For $\mathrm{EES}(2,7;x)$ the explicit closed form of the $2N$
coefficients involves $\sqrt{2}$ and quartic polynomials in $x$, and is
most compactly expressed in terms of the Butcher entries of the
$4$-stage tableau itself~\citep{shmelev2025explicit}: for either sign
choice $\pm\sqrt{2}$,
\begin{align*}
B_1 &= a_{21}, & B_2 &= a_{32}, & B_3 &= a_{43}, & B_4 &= b_4, \\
A_2 &= \frac{a_{31} - a_{21}}{a_{32}}, & A_3 &= \frac{a_{42} - a_{32}}{a_{43}}, & A_4 &= \frac{b_3 - a_{43}}{b_4},
\end{align*}
where $(a_{ij}, b_i)$ are the entries of the $\mathrm{EES}(2,7;x)$
Butcher tableau. At $x = \tfrac{1}{14}(5-3\sqrt{2})$ with the
$+\sqrt{2}$ branch these evaluate to
$(B_1, B_2, B_3, B_4) = \bigl(\tfrac{2-\sqrt{2}}{3}, \tfrac{4+\sqrt{2}}{8}, \tfrac{3(3-\sqrt{2})}{7}, \tfrac{9-4\sqrt{2}}{14}\bigr)$
and
$(A_2, A_3, A_4) = \bigl(\tfrac{-7+4\sqrt{2}}{3}, -\tfrac{4+5\sqrt{2}}{12}, \tfrac{3(-31+8\sqrt{2})}{49}\bigr)$.

Unrolling the Williamson recurrence as in \citet{bazavov_commutator-free_2020}
turns the $2N$ coefficients $(A_i, B_i)$ above into explicit weight vectors
$\beta_l \in \mathbb{R}^s$ that specify the arguments of each exponential in
the commutator-free lift. This makes the induced homogeneous-space integrator
$\mathrm{CF\text{-}EES}(2,5;\tfrac{1}{10})$ fully explicit.

\begin{proposition}[$\mathrm{CF\text{-}EES}(2,5;\tfrac{1}{10})$ tableau]
\label{prop:cf-ees-25-tableau}
Let $\mathcal{M} = G/H$ be a homogeneous space with transitive action
$\Lambda : G \times \mathcal{M} \to \mathcal{M}$ and generator map
$\xi : \mathcal{M} \to \mathfrak{g}$
(as in Section~\ref{subsec:homogeneous_space_integration}). One step of
$\mathrm{CF\text{-}EES}(2,5;\tfrac{1}{10})$ from $y_n \in \mathcal{M}$
to $y_{n+1}$, with $Y_0 := y_n$, is
\begin{equation}
\label{eq:cfees25-recurrence}
K_l = \xi(Y_{l-1}),\qquad
V_l = h \sum_{i=1}^{l} \beta_{l,i}\, K_i,\qquad
Y_l = \Lambda\!\bigl(\exp(V_l),\, Y_{l-1}\bigr),
\qquad l = 1, 2, 3,
\end{equation}
with $y_{n+1} := Y_3$ and weight vectors given by
\[
\renewcommand{\arraystretch}{1.35}
\begin{array}{c|ccc}
 & i=1 & i=2 & i=3 \\ \hline
\beta_{1,i} & \tfrac{1}{3} & 0 & 0 \\
\beta_{2,i} & -\tfrac{7}{16} & \tfrac{15}{16} & 0 \\
\beta_{3,i} & \tfrac{49}{240} & -\tfrac{7}{16} & \tfrac{2}{5} \\ \hline
\sum_{l} \beta_{l,i} = b_i & \tfrac{1}{10} & \tfrac{1}{2} & \tfrac{2}{5}
\end{array}
\]
Row $l$ gives the coefficients of $K_1, \ldots, K_s$ inside the $l$th
exponential $\exp(V_l)$; the final row is the Euclidean consistency check
$\sum_l \beta_{l,i} = b_i$, so that on a linear manifold
(\S\ref{subsec:homogeneous_space_integration}) the three exponentials collapse
to the classical update $y_{n+1} = y_n + h \sum_i b_i K_i$ of
$\mathrm{EES}(2,5;\tfrac{1}{10})$.
\end{proposition}

\begin{proof}
From the $2N$ recurrence $\delta_l = A_l \delta_{l-1} + h K_l$ with
$\delta_0 = 0$ and $A_1 = 0$, iteration gives
$\delta_l = h \sum_{i \le l} A_l A_{l-1} \cdots A_{i+1} K_i$ (empty product
$=1$). The frozen step $Y_l = \Lambda(\exp(B_l \delta_l), Y_{l-1})$
is therefore the exponential of
$V_l = h \sum_{i \le l} \beta_{l,i} K_i$ with
$\beta_{l,i} = B_l A_l A_{l-1} \cdots A_{i+1}$ for $i < l$,
$\beta_{l,l} = B_l$, and $\beta_{l,i} = 0$ for $i > l$. Substituting
$(A_i) = (0, -\tfrac{7}{15}, -\tfrac{35}{32})$ and
$(B_i) = (\tfrac{1}{3}, \tfrac{15}{16}, \tfrac{2}{5})$ from the preceding
proposition gives the tabulated values. The identity $\sum_l \beta_{l,i} = b_i$
follows by telescoping the unrolling and matches the classical weights
$b = (\tfrac{1}{10}, \tfrac{1}{2}, \tfrac{2}{5})$ of $\mathrm{EES}(2,5;\tfrac{1}{10})$.
\end{proof}

Both recurrences additionally admit a three-register low-storage implementation with a first-order embedded estimator, obtained by storing the final internal stage and advancing it over the remaining fraction of the step by a single Euler update; specifically, one stores the stage at $c_3=\frac{5}{6}$ for $\mathrm{EES}(2,5;1/10)$ and the stage at $c_4=\frac{4+\sqrt{2}}{6}$ for $\mathrm{EES}(2,7;\frac{1}{14}(5-3\sqrt{2}))$ \citep{williamson_low-storage_1980,KETCHESON20101763}. In common with other Diffrax solvers, adaptive stepping requires a fourth auxiliary register holding $y_n$ to permit restart on step rejection.

\section{Order conditions for CF-EES}
\label{sec:cfees-order-conditions}

This section gives a brief account of how the Lie--Butcher (LB) series
character of $\mathrm{CF\text{-}EES}(2,5;x)$ is computed on the
Munthe-Kaas--Wright (MKW) Hopf algebra of planar rooted forests, lists
the closed-form Williamson $2N$ coefficients of the family, and tabulates
the symbolic character $\phi(\tau)$ for every planar tree of order at
most~$5$. From this table the planar order conditions and the
antisymmetric-order conditions of $\mathrm{CF\text{-}EES}(2,5;x)$ can be
read off as identities in~$x$. All expressions in this section are
produced symbolically by the \texttt{kauri} package~\cite{kauri}.

\subsection{General form of \texorpdfstring{$\mathrm{CF\text{-}EES}(2,5;x)$}{CF-EES(2,5;x)}}
\label{subsec:cfees25-general-form}

For completeness, we record the explicit reused-stage form of
$\mathrm{CF\text{-}EES}(2,5;x)$ for arbitrary admissible~$x$. The underlying $\mathrm{EES}(2,5;x)$
Butcher tableau (Proposition~\ref{prop:EES_2_5}) admits the Williamson $2N$
reduction
\begin{equation}
\label{eq:ees25-2n-coeffs}
B_1 = \frac{2x+1}{4(1-x)},\qquad
B_2 = \frac{1-x}{1-4x^2},\qquad
B_3 = \frac{1-2x}{2},
\end{equation}
\begin{equation}
\label{eq:ees25-2n-A}
A_2 = \frac{4x^2-2x+1}{2(x-1)},\qquad
A_3 = -\frac{4x^2-2x+1}{(2x-1)^2(2x+1)},
\end{equation}
valid for $x \in \mathbb R \setminus \{1, \pm \tfrac{1}{2}\}$. Following the Bazavov
$2N$ commutator-free lift~\citep{bazavov_commutator-free_2020}, one step of
$\mathrm{CF\text{-}EES}(2,5;x)$ from $y_n \in \mathcal M$ to $y_{n+1}$ on
a homogeneous space $\mathcal M = G/H$ with action $\Lambda$ and generator
$\xi : \mathcal M \to \mathfrak g$ is the two-register recurrence
\begin{equation}
\label{eq:cfees25-2n-recurrence}
\begin{aligned}
  Y_0 &:= y_n, \qquad \delta_0 := 0, \\
  K_l &= \xi(Y_{l-1}), \\
  \delta_l &= A_l\, \delta_{l-1} + h\, K_l, \\
  Y_l &= \Lambda\bigl(\exp(B_l\, \delta_l),\, Y_{l-1}\bigr),
  \qquad l = 1, 2, 3,
\end{aligned}
\end{equation}
with $A_1 := 0$, and $y_{n+1} := Y_3$. Only the current state
$Y_l \in \mathcal{M}$ and the current stage increment
$\delta_l \in \mathfrak{g}$ are stored at any time, hence the ``$2N$''
descriptor.

Unrolling~\eqref{eq:cfees25-2n-recurrence} expresses the exponential
argument $B_l\delta_l$ as a linear combination of the slopes
$K_1,\ldots,K_l$,
\[
B_l\delta_l
=
h\sum_{i=1}^{l}\beta_{l,i}K_i,
\qquad
\beta_{l,i}
=
B_l A_l A_{l-1}\cdots A_{i+1}\quad (i<l),
\qquad
\beta_{l,l}=B_l,
\]
yielding the generic weight matrix
\begin{equation}
\label{eq:cfees25-betas-general}
\renewcommand{\arraystretch}{1.4}
\begin{array}{c|ccc}
 & i = 1 & i = 2 & i = 3 \\ \hline
\beta_{1,i} & B_1 & 0 & 0 \\
\beta_{2,i} & B_2 A_2 & B_2 & 0 \\
\beta_{3,i} & B_3 A_3 A_2 & B_3 A_3 & B_3 \\ \hline
b_i = \sum_l \beta_{l,i} & x & \tfrac{1}{2} & \tfrac{1}{2} - x
\end{array}
\end{equation}
The bottom row recovers the Butcher weights of $\mathrm{EES}(2,5;x)$, so
that on a flat manifold ($\Lambda(\exp(v), y) = y + v$, $\xi(y) = f(y)$)
the three exponentials collapse to the classical update
$y_{n+1} = y_n + h \sum_i b_i K_i$. At $x = \tfrac{1}{10}$,
\eqref{eq:ees25-2n-coeffs}--\eqref{eq:cfees25-betas-general} reduce to the
numerical coefficients of Proposition~\ref{prop:cf-ees-25-tableau}.

It is the reused-stage feature of~\eqref{eq:cfees25-2n-recurrence}, in
which the same accumulated increment $\delta_l$ supplies the argument of
every exponential and the slope $K_l = \xi(Y_{l-1})$ is evaluated at the
previously \emph{computed} stage value $Y_{l-1}$ rather than at $y_n$,
that introduces the substitution-and-pseudo-stage subtlety in the
LB-series analysis below.

\subsection{LB character of CF-EES via Owren's pseudo-stage construction}
\label{subsec:cfees-character-construction}

Each of the three exponentials $\exp(V_l)$ ($l=1,2,3$) in one step of
$\mathrm{CF\text{-}EES}(2,5;x)$ has an associated single-exponential
character $\phi_{l,\mathrm{exp}} : \mathcal H_{\mathrm{MKW}} \to \mathbb R$,
defined on a planar tree by the shuffle-symmetric formula
$\phi_{l,\mathrm{exp}}([\tau_1,\ldots,\tau_k])
= \tfrac{1}{k!}\prod_i \mathcal V_l(\tau_i)$, with
$\mathcal V_l(\tau) = \sum_{i=1}^{s} \beta_{l,i}\, g_i(\tau)$ encoding the
fact that the argument of the $l$-th exponential is a linear combination
of the slopes $K_i$, each of which has its own row character $g_i$
described below.

Because $\mathrm{CF\text{-}EES}(2,5;x)$ is a reused-stage method, the
slope $K_l = \xi(Y_{l-1})$ is evaluated at the previously computed stage
value $Y_{l-1}$, so the LB character of $K_l$ is itself a Lie--Butcher
series that must be propagated through the construction. Following
\citet[Theorem~2.5]{owren2006order}, the row character $g_l$ at stage~$l$
satisfies the substitution recursion
\begin{equation}
\label{eq:cfees-row-recursion}
  g_l\bigl([\tau_1,\ldots,\tau_k]\bigr)
  = \sum_{j=0}^{k}
      g_l\bigl(B_+(\tau_1,\ldots,\tau_j)\bigr)
      \cdot \phi_{l,\mathrm{exp}}\bigl(B_+(\tau_{j+1},\ldots,\tau_k)\bigr),
\end{equation}
with base case $g_l(\bullet) = 0$ for $l = 1$ and $g_l(\bullet) = 1$ for
$l \geq 2$ (so that the first stage is the identity and each later stage
contributes one additional exponential). The output of the method is the
result of applying the final exponential to the previous stage value, and
its LB character is obtained by treating this final update as a
\emph{pseudo-stage}: on a planar tree $\tau$, one grafts $\tau$ to a new
root via $B_+$ and evaluates the final row character $g_3$ at the
result,
\begin{equation}
\label{eq:cfees-pseudostage}
  \phi(\tau) := g_3\bigl(B_+(\tau)\bigr).
\end{equation}

\subsection{The LB character of \texorpdfstring{$\mathrm{CF\text{-}EES}(2,5;x)$ on planar trees of order $\leq 5$}{CF-EES(2,5;x) on planar trees of order <= 5}}
\label{subsec:alpha-cf-ees25-table}

Substituting~\eqref{eq:ees25-2n-coeffs}--\eqref{eq:ees25-2n-A} into the
pseudo-stage construction~\eqref{eq:cfees-pseudostage} gives, for every
planar rooted tree $\tau$ of order at most~$5$, an explicit rational
function $\phi(\tau) \in \mathbb Q[x]$. Table~\ref{tab:alpha-cf-ees25-x}
lists every value, organised by tree order.
The first  three rows ($|\tau| \leq 2$) match $1/\tau!$ identically in~$x$, confirming the classical order condition; deeper rows give the
$x$-dependent values that encode the antisymmetric-order cancellations proven below.

\begin{table}[ht]
\centering
\caption{LB character $\phi(\tau)$ of $\mathrm{CF\text{-}EES}(2,5;x)$ on
every planar rooted tree $\tau$ with $|\tau| \leq 5$, computed from the
pseudo-stage construction~\eqref{eq:cfees-pseudostage} with the symbolic
Williamson coefficients
\eqref{eq:ees25-2n-coeffs}--\eqref{eq:ees25-2n-A}.}
\label{tab:alpha-cf-ees25-x}
\renewcommand{\arraystretch}{2.2}
\setlength{\tabcolsep}{2.5pt}
\begin{tabular}{c@{\hspace{0.5em}}l@{\hspace{2.0em}}
                c@{\hspace{0.5em}}l@{\hspace{2.0em}}
                c@{\hspace{0.5em}}l}
\toprule
$\tau$ & $\phi(\tau)$ &
$\tau$ & $\phi(\tau)$ &
$\tau$ & $\phi(\tau)$ \\
\midrule
\multicolumn{6}{l}{\textit{Orders 0--2 (matching $1/\tau!$ identically in $x$):}}\\
\TEmpty & $1$ &
\TLeaf & $1$ &
\TChainTwo & $\tfrac{1}{2}$ \\[0.2em]
\midrule
\multicolumn{6}{l}{\textit{Order 3:}}\\
\TTwoLeaves & $\tfrac{2x-5}{32(x-1)}$ &
\TChainThree & $\tfrac{1}{8}$ &
& \\[0.2em]
\midrule
\multicolumn{6}{l}{\textit{Order 4:}}\\
\TThreeLeaves & $-\tfrac{2x+7}{192(x-1)}$ &
\TLeafChainTwo & $-\tfrac{x+2}{32(x-1)}$ &
\TChainTwoLeaf & $\tfrac{1}{32}$ \\
\TNestedTwoLeaves & $-\tfrac{2x+1}{64(x-1)}$ &
\TChainFour & $0$ &
& \\[0.2em]
\midrule
\multicolumn{6}{l}{\textit{Order 5:}}\\
\TFourLeaves & $\tfrac{8x^3+24x^2+36x-41}{6144(x-1)^3}$ &
\TLeafLeafChainTwo & $\tfrac{4x^2+10x+13}{768(x-1)^2}$ &
\TLeafChainTwoLeaf & $-\tfrac{4x+5}{384(x-1)}$ \\
\TLeafNestedTwoLeaves & $\tfrac{(2x+1)(x+2)}{256(x-1)^2}$ &
\TLeafChainThree & $0$ &
\TChainTwoLeafLeaf & $\tfrac{1}{192}$ \\
\TChainTwoChainTwo & $-\tfrac{1}{64(2x-1)}$ &
\TNestedTwoLeavesLeaf & $-\tfrac{2x+1}{256(x-1)}$ &
\TNestedThreeLeaves & $\tfrac{(2x+1)^2}{768(x-1)^2}$ \\
\TNestedLeafChainTwo & $0$ &
\TChainThreeLeaf & $0$ &
\TNestedChainTwoLeaf & $0$ \\
\TDoubleNestedTwoLeaves & $0$ &
\TChainFive & $0$ &
& \\
\bottomrule
\end{tabular}
\end{table}

\subsection{Order theorem}
\label{subsec:cfees25-order-theorem}

\begin{theorem}\label{thm:cf-ees25-orders}
For every $x \in \mathbb R \setminus \{1, \pm \tfrac{1}{2}\}$, the LB
character $\phi$ of $\mathrm{CF\text{-}EES}(2,5;x)$ defined
by~\eqref{eq:cfees-pseudostage} satisfies, on $\mathcal H_{\mathrm{MKW}}$:
\begin{enumerate}
  \item \textbf{Planar order $2$:} $\phi(\tau) = 1/\tau!$ for every
  planar $\tau$ with $|\tau| \leq 2$.
  \item \textbf{Antisymmetric order $5$:} the symmetric defect
  $D := (\mathrm{sign}\cdot\phi) \star_{\mathrm{MKW}} \phi$, where
  $\mathrm{sign}(\tau) = (-1)^{|\tau|}$, satisfies $D(\tau) = \varepsilon(\tau)$
  for every planar $\tau$ with $|\tau| \leq 5$.
\end{enumerate}
\end{theorem}

The same construction applies to $\mathrm{CF\text{-}EES}(2,7;x)$. The
explicit symbolic table for that family contains $1+1+2+5+14+42+132 = 197$
entries with $\sqrt 2$ throughout the rational expressions and is too
large to reproduce here, but the analogous statement holds:

\begin{theorem}\label{thm:cf-ees27-orders}
For every $x$ in the admissible domain of $\mathrm{EES}(2,7;x)$ and either
sign choice $\pm\sqrt{2}$ in its tableau, $\mathrm{CF\text{-}EES}(2,7;x)$
has planar order $2$ and antisymmetric order $7$ on $\mathcal H_{\mathrm{MKW}}$.
\end{theorem}

\section{RDEs on Homogeneous spaces}
\label{sec:cfees-rde-extension}

The order conditions verified in
Section~\ref{sec:cfees-order-conditions} were stated as algebraic
identities on the LB character of $\mathrm{CF\text{-}EES}(2,5;x)$. We sketch here why these identities yield the corresponding local error
rates when $\mathrm{CF\text{-}EES}(2,5;x)$ is applied to a rough
differential equation
\begin{equation}
\label{eq:cfees-rde}
  dy_t = \sum_{i=1}^{d}\, \xi_i(y_t)_{\mathcal M}\, dX^i_t,
  \qquad y_s = y \in \mathcal M = G/H,
\end{equation}
driven by an $\alpha$-H\"older planarly branched rough path $\mathbb X$
above a control $X \in C^\alpha([0,T]; \mathbb R^d)$ in the sense
of~\citet{curry2020planarly}. The argument splits cleanly into two
parts: the forward (classical) order condition requires a comparison
with the exact rough-path expansion of \citet{curry2020planarly}, while
the antisymmetric-order condition is purely algebraic and inherits
unchanged from Section~\ref{sec:cfees-order-conditions}.

\subsection{Forward order}
\label{app:cfees_fwd_order}

By \citet[Theorem~8.1]{curry2020planarly}, the exact solution of
\eqref{eq:cfees-rde} admits the LB-series representation
\begin{equation}
\label{eq:cfees-rde-exact}
  \mathbb Y_{st}
  = \sum_{\tau \in \mathcal F^{\mathrm{pl}}_A}
    \langle\, \mathbb X_{st} \circ \mathfrak{a}^{\ll},\, \tau\,\rangle\, \tau,
\end{equation}
where the sum is over $A$-decorated planar rooted forests and
$\mathfrak{a}^{\ll}: \mathcal H_{\mathrm{MKW}} \to \mathcal H_{\sshu}$ is
the planar arborification, the Hopf-algebra morphism that maps each
planar tree to the formal sum of its linear extensions. The exact LB
character of the RDE on planar trees is therefore
\begin{equation*}
  e_\mathbb X(\tau) := \langle\, \mathbb X_{st} \circ \mathfrak{a}^{\ll},\, \tau\,\rangle,
\end{equation*}
which is multiplicative in the shuffle sense and depends on the driving rough path $\mathbb X$.

Applying $\mathrm{CF\text{-}EES}(2,5;x)$ to~\eqref{eq:cfees-rde} via the
simplified Redmann--Riedel-style embedding (each $h K_i$ replaced by
$\sum_l X^l_{0,h}\, K^l_i$) gives a one-step LB character of the form
\begin{equation*}
  \phi_\mathbb X(\tau) = \phi(\tau) \cdot \langle X_{0,h},\, \tau\,\rangle,
  \qquad
  \langle X_{0,h},\, \tau\,\rangle := \prod_{v \in V(\tau)} X^{\ell(v)}_{0,h},
\end{equation*}
where $\phi(\tau)$ is the path-independent character of
Section~\ref{subsec:alpha-cf-ees25-table} and $\ell(v)$ is the label of
vertex $v$. For a \emph{geometric} planarly branched rough path, the
planar arborification factorises as
$\langle\, \mathbb X_{st} \circ \mathfrak{a}^{\ll},\, \tau\,\rangle
= \langle X_{0,h}, \tau\rangle / \tau!$, so the local
order condition $\phi_\mathbb X(\tau) = e_\mathbb X(\tau)$ on every
planar $\tau$ with $|\tau| \leq p$ is equivalent to the algebraic
identity
\begin{equation*}
  \phi(\tau) = 1/\tau!, \qquad |\tau| \leq p.
\end{equation*}
By Theorem~\ref{thm:cf-ees25-orders}\,(i) this holds for $|\tau| \leq 2$,
identically in $x$, so $\mathrm{CF\text{-}EES}(2,5;x)$ has local error
of order $(2+1)\alpha = 3\alpha$ when applied to~\eqref{eq:cfees-rde};
the same argument with Theorem~\ref{thm:cf-ees27-orders} gives local
order $3\alpha$ for $\mathrm{CF\text{-}EES}(2,7;x)$.

\subsection{Backward order}
\label{app:cfees_bwd_order}
The symmetric defect
$D := (\mathrm{sign}\cdot\phi) \star_{\mathrm{MKW}} \phi$ is the LB
character of the composition $\Phi^{\mathbb X^{\mathrm{rev}}} \circ
\Phi^{\mathbb X}$, that is, one step of $\mathrm{CF\text{-}EES}(2,5;x)$
under~$\mathbb X$ followed by one step under the time-reversed driver
$\mathbb X^{\mathrm{rev}}$. $D(\tau)$ can be computed entirely
from the Butcher tableau of $\mathrm{EES}(2,5;x)$ and the MKW
coproduct, without depending on the driving rough path. The path-applied character of the
composition factorises in the same way as the forward character,
\begin{equation*}
  D_\mathbb X(\tau)
  = D(\tau) \cdot \langle X_{0,h},\, \tau\,\rangle,
\end{equation*}
so that vanishing of $D(\tau)$ at the algebra level implies vanishing
of $D_\mathbb X(\tau)$ for any rough path $\mathbb X$ and any tree of
the same order. Theorem~\ref{thm:cf-ees25-orders}\,(ii) gives
$D(\tau) = \varepsilon(\tau)$ on every planar $\tau$ with $|\tau| \leq 5$,
identically in $x$. Hence
$\Phi^{\mathbb X^{\mathrm{rev}}} \circ \Phi^{\mathbb X}$ recovers the
identity to local order $(5+1)\alpha = 6\alpha$, for any geometric
planarly branched rough path. The same argument with
Theorem~\ref{thm:cf-ees27-orders} gives local order $8\alpha$ for
$\mathrm{CF\text{-}EES}(2,7;x)$.

\subsection{Backpropagation through homogeneous-space 2N commutator-free methods}
\label{app:backprop_through_homogeneous_cf_methods}
The algorithm for backpropagation through a homogeneous-space
commutator-free scheme is given in Algorithm~\ref{alg:homogeneous_cf_backprop}.
We assume that the solver is at least approximately reversible, so the backward pass
recovers the preceding stage states and the low-storage Lie-algebra
register by applying the reverse recurrence. Unlike the Euclidean case,
the adjoints with respect to stage states are covectors
$\lambda_{Y_l} \in T^*_{Y_l}M$, and each reverse stage applies the
pullback of the homogeneous-space action. Thus the backward sweep may be
viewed as a discrete evolution on the cotangent bundle $T^*M$.

For
\[
    \Psi_l(Y,\delta)
    =
    \Lambda\!\left(\exp(B_l\delta),Y\right),
\]
we write $D_Y\Psi_l^*$ and $D_\delta\Psi_l^*$ for the adjoints of the
differentials with respect to the state and Lie-algebra arguments.

\begin{algorithm}[H]
\caption{Backpropagation through Homogeneous-Space $2N$ Commutator-Free Schemes}
\label{alg:homogeneous_cf_backprop}
\textbf{Input:} $y_{n+1}$, $\lambda_{Y_s} = \partial_{y_{n+1}}L$\\
\textbf{Input:} Running derivative with respect to $\theta$, $\partial_\theta L$\\
\textbf{Input:} Homogeneous-space commutator-free method $\Phi$ with $\{A_l\}_{l=1}^{s-1}$ and $\{B_l\}_{l=1}^s$.
\begin{algorithmic}

\State $Y_s = y_{n+1}$, $\lambda_\delta = 0$

\For {$l = s, \ldots, 1$}
    \State Recover $Y_{l-1}$, $\delta_l$, and $K_l$ by the algebraic reverse recurrence $\Phi^\mathrm{rev}$
    \State $\lambda_{Y_{l-1}} = D_Y\Psi_l(Y_{l-1},\delta_l)^* \lambda_{Y_l}$
    \State $\lambda_{\delta_l} \mathrel{+}= D_\delta\Psi_l(Y_{l-1},\delta_l)^* \lambda_{Y_l}$
    \State $\lambda_{K_l} = \lambda_{\delta_l}$
    \State $d_\theta, \eta_l = \mathrm{backprop}_{\xi}(dX,\lambda_{K_l})$
    \State $\lambda_{Y_{l-1}} \mathrel{+}= \eta_l$
    \State $\partial_\theta L \mathrel{+}= d_\theta$
    \State $\lambda_{\delta_{l-1}} = A_l\lambda_{\delta_l}$
\EndFor

\State $\partial_{y_n}L = \lambda_{Y_0}$\\
\Return $y_n, \partial_{y_n}L, \partial_\theta L$

\end{algorithmic}
\end{algorithm}

\section{Convergence experiments}
\label{app:convergence_experiments}

We verify the global error rates given in Theorem~\ref{thm:ees_global_error} experimentally by reproducing the example given in ~\citep{redmann2020runge, deya2012milstein} for $\mathrm{EES}(2,5; 1/10)$ and $\mathrm{EES}(2,7; (5 - 3\sqrt{2}) / 14)$. We take the RDE
\begin{equation*}
    dy_t = \cos(y_t) d\mathbf{X}^{(1)}_t + \sin(y_t)d\mathbf{X}^{(2)}_t, \quad y_0 = 1
\end{equation*}
for $t \in [0,1]$, where $\mathbf{X}$ is the geometric lift of a 2-dimensional fractional Brownian motion (fBm) with Hurst index $H$. We compute the average of the maximal discretisation error over $M = 10$ realisations of the RDE,
\begin{equation*}
    \mathcal{E}(h) := \frac{1}{M}\sum_{i=1}^M \max_{n=0,\ldots, N} |y_i(t_n) - y_{i,n}|,
\end{equation*}
where $y_i(t)$ denotes the solution to the $i^{th}$ realisation of the RDE and $y_{i,n}$ denotes the discretisation of the $i^{th}$ solution using an $\mathrm{EES}$ scheme. Additionally, we evaluate the average error when recovering the initial condition,
\begin{equation*}
    \backvec{\mathcal{E}}(h) := \frac{1}{M}\sum_{i=1}^M |y_0 - \backvec{y}_{i,n}|.
\end{equation*}
From ~\citep{friz2014convergence}, the rate $r_0$ can be chosen arbitrarily close to $2H - 1/2$ for a fractional Brownian motion with Hurst parameter $H$. It follows that we expect $\eta_1 = 2H - 1/2$ in Theorem \ref{thm:ees_global_error} for both $\mathrm{EES}(2,5)$ and $\mathrm{EES}(2,7)$, and $\eta_2 = 6H-1$ for $\mathrm{EES}(2,5)$ and $\eta_2 = 8H-1$ for $\mathrm{EES}(2,7)$. We show the rates for $H \in (0.4, 0.5, 0.6)$ in Figure~\ref{fig:ees_convergence_12}.

The same convergence orders are observed for $\mathrm{CF\text{-}EES}(2,5)$ and
$\mathrm{CF\text{-}EES}(2,7)$; see Figure~\ref{fig:cfees_convergence_12}.
We consider the $\mathrm{SO}(3)$ rough differential equation
\[
dX_t = \sum_{a=1}^{2} X_t\,\xi_a(X_t)\,d\mathbf{X}^a_t,
\qquad X_0 = I,
\]
driven by a two-dimensional fractional Brownian motion $\mathbf{X}$ with Hurst parameter
$H \in \{0.4, 0.5, 0.6\}$. Writing $X=(X_{ij})_{i,j=1}^3$, the coefficient maps
$\xi_a:\mathrm{SO}(3)\to\mathfrak{so}(3)$ are affine in the entries of $X$:
\begin{equation*}
\begin{aligned}
\xi_1(X) &=
\begin{pmatrix}
0 & -0.1 - 0.3X_{31} & \phantom{-}0.25 + 0.2X_{23} \\
0.1 + 0.3X_{31} & 0 & -0.9 - 0.2X_{11} \\
-0.25 - 0.2X_{23} & \phantom{-}0.9 + 0.2X_{11} & 0
\end{pmatrix}, \\
\xi_2(X) &=
\begin{pmatrix}
0 & -0.8 - 0.15X_{33} & -0.35 + 0.2X_{22} \\
0.8 + 0.15X_{33} & 0 & -0.15 - 0.25X_{12} \\
0.35 - 0.2X_{22} & \phantom{-}0.15 + 0.25X_{12} & 0
\end{pmatrix}.
\end{aligned}
\end{equation*}

\begin{figure}[ht]
\centering
\includegraphics[
    width=\textwidth,
]{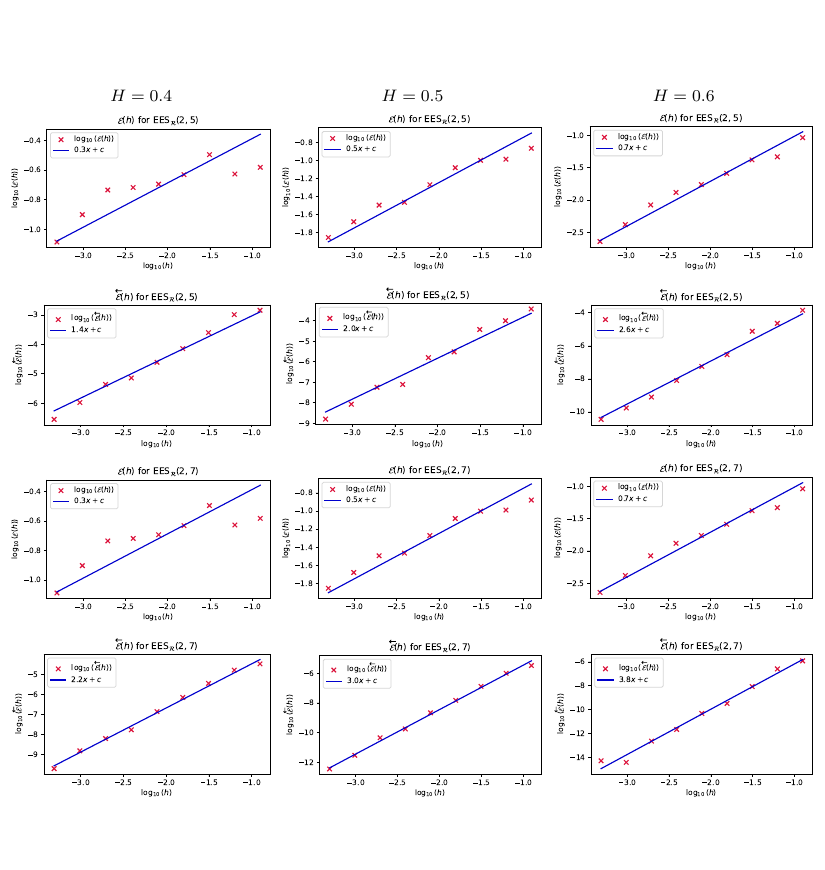}
\vspace{-3em}
\caption{\textit{Convergence rates for $\mathrm{EES}$ with $H \in (0.4, 0.5, 0.6)$}}
\label{fig:ees_convergence_12}
\end{figure}
\vspace{-2em}

\begin{figure}[ht]
\centering
\includegraphics[
    width=\textwidth,
]{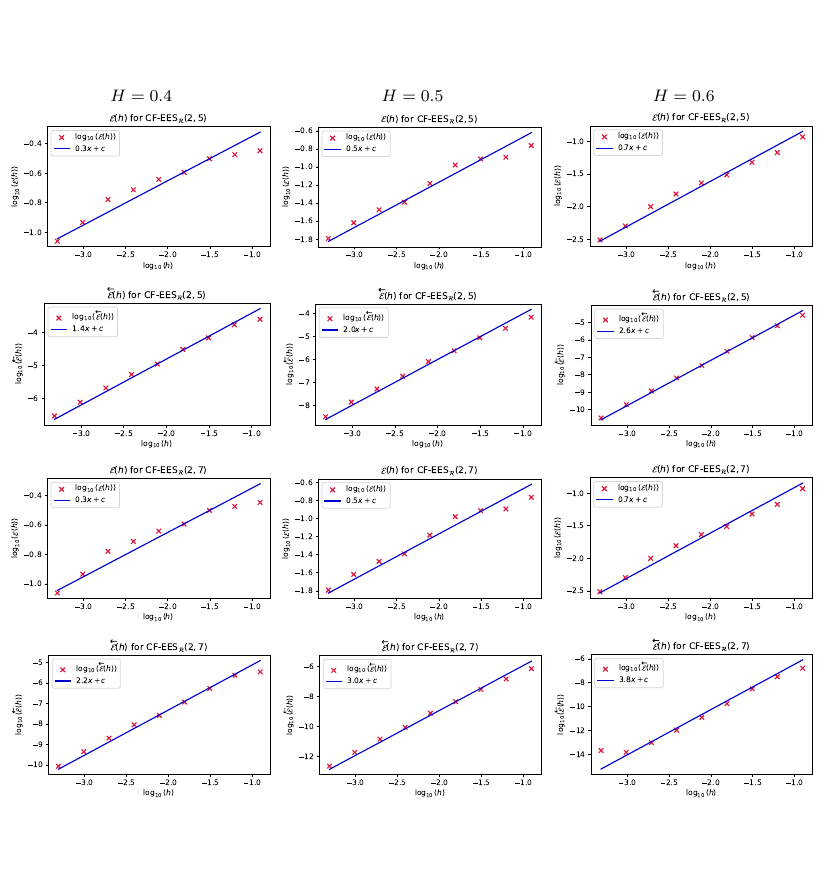}
\vspace{-3em}
\caption{\textit{Convergence rates for $\mathrm{CF\text{-}EES}$ with $H \in (0.4, 0.5, 0.6)$}}
\label{fig:cfees_convergence_12}
\end{figure}

\begin{figure}[ht]
\centering
\includegraphics[width=0.5\textwidth]{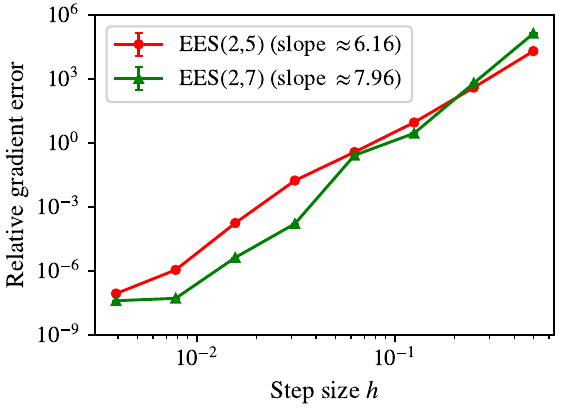}
\caption{\textit{The higher order of $\mathrm{EES(2,7)}$ is nullified by instability due to non-smooth NSDE vector fields at practical step sizes.}}
\label{fig:ees25_vs_27_grad_error}
\end{figure}

\clearpage

\section{Additional Experiments}
\label{app:additional_experiments}
\subsection{High-dimensional GBM with stiff drift}

Consider learning the dynamics of a high-dimensional geometric Brownian motion (GBM)
\begin{equation*}
    dy_t = A y_t dt + \sigma y_t dW_t, \quad y_0 \in \mathbb{R}^d,
\end{equation*}
where $A \in \mathbb{R}^{d \times d}$ and $\sigma \in \mathbb{R}$. We introduce a stiff drift component by choosing $A = QDQ^T$, where $D = \text{diag}(\lambda_0, \lambda_1, \ldots, \lambda_{d-1})$, $\lambda_i = -20(1 + \frac{i}{d})$, and $Q$ is a randomly generated orthogonal matrix, and take $d = 25$ and $\sigma = 0.1$. We choose to learn the dynamics using a Neural SDE of the form
\begin{equation*}
    dz_t = g(z_t; \theta_g) dt + f(z_t; \theta_f) \circ dW_t, \quad z_0 = h(\mathbf{x}, \theta_h) \in \mathbb{R}^{d_z},
\end{equation*}
where $f,g$ are neural networks with the same architecture as in the OU experiment of Section~\ref{sec:experiments}. We integrate the SDEs over $t \in [0,1]$ for 1,000 epochs, sampling 10,000 realisations of the dynamics at every epoch. The Adam optimiser is used with a fixed learning rate of $2\times 10^{-1}$.\par\smallskip

As in the previous example, Table~\ref{table:gbm} and Figure~\ref{fig:gbm} show the results of training using various reversible methods, with the step size chosen such that the number of evaluations of $f,g$ is fixed. We see that the instability caused by the stiff drift results in diverging MSE for all solvers except $\mathrm{EES}(2,5)$, which manages to retain moderate stability for the entire 1,000 epochs of training.

\begin{table}[ht]
\centering
    \caption{Metrics for stiff GBM dynamics. The step size is chosen such that the total number of evaluations of $f,g$ per integration is fixed.}
    \label{table:gbm}
    \renewcommand{\arraystretch}{1.1}
    \begin{tabular}{lcccc}
        \toprule
        Method & \# Eval. / Step & Step Size & Terminal MSE & Runtime (s) \\
        \midrule
        Reversible Heun & 1 & $1/60$ & -- & 1283.6 \\
        MCF Euler & 2 & $1/30$ & -- & 1119.9 \\
        MCF Midpoint & 4 & $1/15$ & -- & 1270.1 \\
        $\mathrm{EES}(2,5)$ & 3 & $1/20$ & $1.1803 \times 10^{-4}$ & 1050.0 \\
        \bottomrule
    \end{tabular}
\end{table}

Figure \ref{fig:gbm_grad} shows the MSE of the gradient of the loss during training, where the true gradient is computed by autodifferentiation through a discretise-then-optimise solution, using the same solver and step size. Despite its near-reversibility, $\mathrm{EES}(2,5)$ achieves a lower gradient MSE compared to other solvers. This effect is likely the result of the superior stability of $\mathrm{EES}(2,5)$, in combination with the linear nature of the target SDE.

\begin{figure}[ht]
\centering
\begin{minipage}{0.45\textwidth}
  \includegraphics[width = \textwidth,trim={0mm 4mm 0mm 0mm},clip]{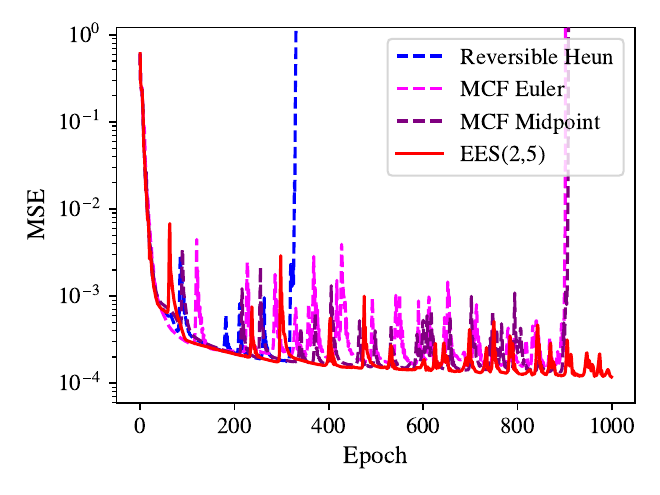}
    \captionsetup{oneside,margin={0.5cm,0cm}}
    \caption{\textit{Training MSE for GBM dynamics with a fixed number of evaluations of $f,g$.}}
    \label{fig:gbm}
\end{minipage}%
\begin{minipage}{0.45\textwidth}
  \includegraphics[width = \textwidth,trim={0mm 4mm 0mm 0mm},clip]{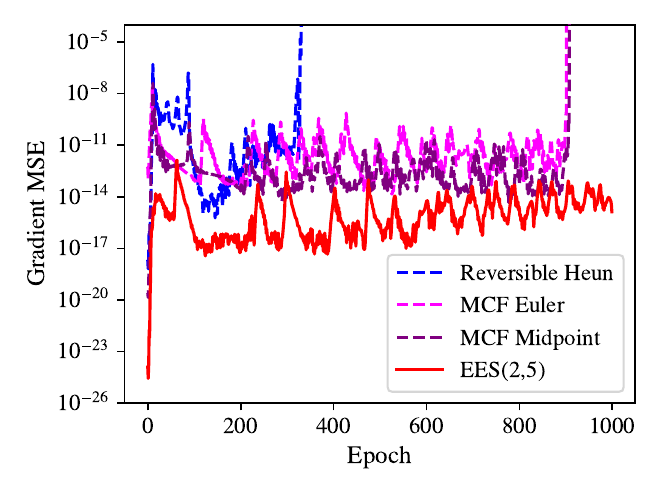}
    \captionsetup{oneside,margin={0.5cm,0cm}}
    \caption{\textit{Gradient MSE for GBM dynamics with a fixed number of evaluations of $f,g$.}}
    \label{fig:gbm_grad}
\end{minipage}
\end{figure}

\clearpage
\subsection{Stochastic Volatility}
\label{app:further_stoch_vol}

We present results for the remaining stochastic volatility models in Table~\ref{tab:further_stoch_vol}. $\mathrm{EES}(2,5)$ remains the lowest runtime integrator, while maintaining identical terminal MSE performance.

\begin{table}[ht]
\centering
\caption{Metrics for benchmark stochastic volatility models.}
\label{tab:further_stoch_vol}
{\renewcommand{\arraystretch}{1.1}
\setlength{\tabcolsep}{3pt}
\begin{tabular}{@{}llcccc@{}}
\toprule
Model & Method & \#Eval. / Step & Step Size & Terminal MSE ($\pm 2\sigma$) & Runtime (s) \\
\midrule
\multirow{4}{*}{Black--Scholes}
  & Reversible Heun & 1 & $1/504$ & $6.46 \scriptstyle{\pm 1.01}$ & 999.8 \\
  & MCF Euler & 2 & $1/252$ & $6.46 \scriptstyle{\pm 1.01}$ & 937.1 \\
  & MCF Midpoint & 4 & $1/126$ & $6.44 \scriptstyle{\pm 1.01}$ & 514.9 \\
  & $\mathrm{EES}(2,5)$ & 3 & $1/168$ & $6.44 \scriptstyle{\pm 1.01}$ & 405.6 \\
\addlinespace

\multirow{4}{*}{Classical Bergomi}
  & Reversible Heun & 1 & $1/504$ & $19.97 \scriptstyle{\pm 1.39}$ & 1001.2 \\
  & MCF Euler & 2 & $1/252$ & $19.97 \scriptstyle{\pm 1.39}$ & 922.7 \\
  & MCF Midpoint & 4 & $1/126$ & $19.98 \scriptstyle{\pm 1.39}$ & 522.3 \\
  & $\mathrm{EES}(2,5)$ & 3 & $1/168$ & $19.97 \scriptstyle{\pm 1.39}$ & 401.8 \\
\addlinespace

\multirow{4}{*}{Local stoch vol}
  & Reversible Heun & 1 & $1/504$ & $6.66 \scriptstyle{\pm 1.02}$ & 1010.5 \\
  & MCF Euler & 2 & $1/252$ & $6.66 \scriptstyle{\pm 1.02}$ & 924.7 \\
  & MCF Midpoint & 4 & $1/126$ & $6.64 \scriptstyle{\pm 1.02}$ & 521.1 \\
  & $\mathrm{EES}(2,5)$ & 3 & $1/168$ & $6.64 \scriptstyle{\pm 1.02}$ & 414.0 \\
\addlinespace

\multirow{4}{*}{Heston}
  & Reversible Heun & 1 & $1/504$ & $5.68 \scriptstyle{\pm 0.82}$ & 986.1 \\
  & MCF Euler & 2 & $1/252$ & $5.69 \scriptstyle{\pm 0.82}$ & 928.0 \\
  & MCF Midpoint & 4 & $1/126$ & $5.69 \scriptstyle{\pm 0.82}$ & 521.7 \\
  & $\mathrm{EES}(2,5)$ & 3 & $1/168$ & $5.69 \scriptstyle{\pm 0.82}$ & 409.1 \\
\addlinespace

\multirow{4}{*}{Rough Heston}
  & Reversible Heun & 1 & $1/504$ & $6.54 \scriptstyle{\pm 2.16}$ & 983.4 \\
  & MCF Euler & 2 & $1/252$ & $6.54 \scriptstyle{\pm 2.16}$ & 914.8 \\
  & MCF Midpoint & 4 & $1/126$ & $6.55 \scriptstyle{\pm 2.16}$ & 511.8 \\
  & $\mathrm{EES}(2,5)$ & 3 & $1/168$ & $6.55 \scriptstyle{\pm 2.16}$ & 402.8 \\
\addlinespace

\multirow{4}{*}{\shortstack{Quadratic\\rough Heston}}
  & Reversible Heun & 1 & $1/504$ & $5.38 \scriptstyle{\pm 1.21}$ & 977.2 \\
  & MCF Euler & 2 & $1/252$ & $5.39 \scriptstyle{\pm 1.21}$ & 866.2 \\
  & MCF Midpoint & 4 & $1/126$ & $5.39 \scriptstyle{\pm 1.21}$ & 478.8 \\
  & $\mathrm{EES}(2,5)$ & 3 & $1/168$ & $5.39 \scriptstyle{\pm 1.21}$ & 410.7 \\
\bottomrule
\end{tabular}}
\end{table}

\subsection{Molecular Dynamics}
\label{app:further_md}
\begin{wrapfigure}{R}{0.38\textwidth}
    \vspace{-1em}
    \includegraphics[width = 0.38\textwidth]{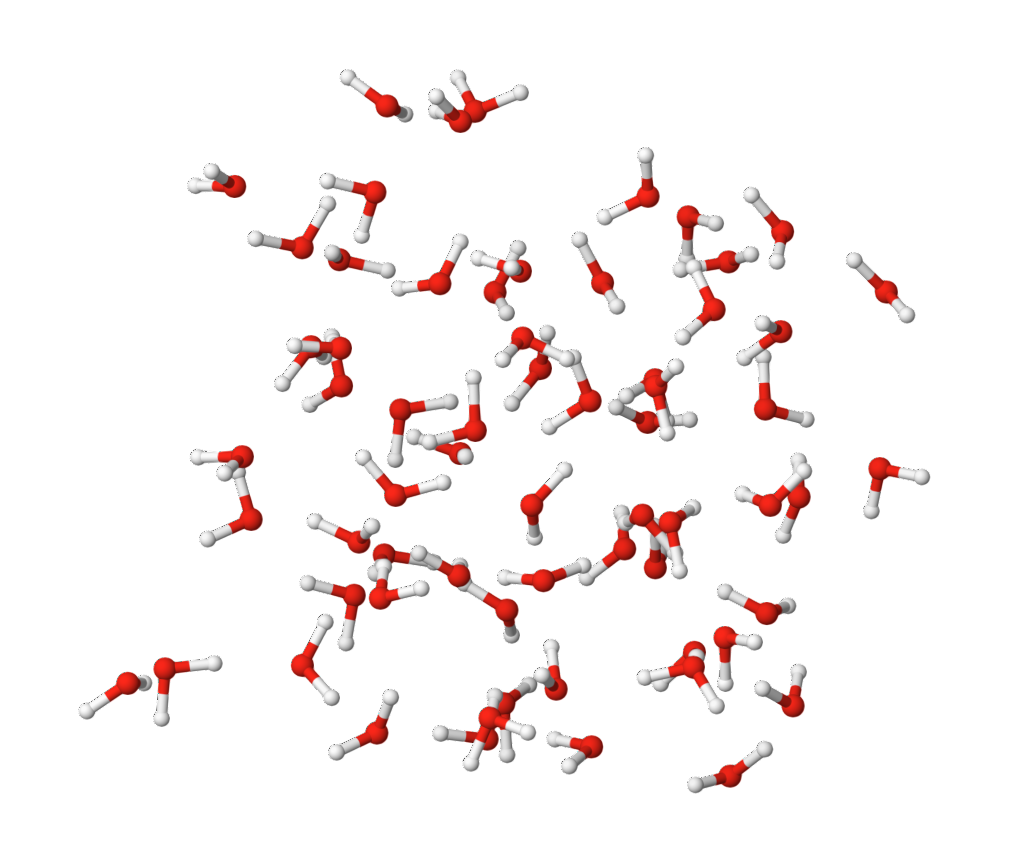}
    \caption{\textit{Initial configuration of the 64-molecule water system}}
    \label{fig:md_water_init}
\vspace{-1.5em}
\end{wrapfigure}
We investigate whether $\mathrm{EES}(2,5)$ can train a neural molecular dynamics force field from long Langevin rollouts under a fixed force-evaluation budget while retaining low-memory adjoint differentiation. This setting is motivated by force-field objectives defined on trajectory-level ensemble observables, such as vibrational spectra, rather than only on static energy-surface quantities. The central difficulty is therefore to obtain useful gradients through long stochastic trajectories without storing the full rollout.

We benchmark on a 64-molecule (192-atom) water system with a pre-trained embedded atom neural network (EANN) force field \cite{zhangEmbeddedAtomNeural2019}. The Langevin state is $y_t = (r_t, v_t) \in \mathbb{R}^{6N_{\mathrm{atom}}} = \mathbb{R}^{1152}$, where $r_t \in \mathbb{R}^{3N_{\mathrm{atom}}}$ and $v_t \in \mathbb{R}^{3N_{\mathrm{atom}}}$ denote the atomic positions and velocities at time $t$. All solvers are compared under matched total evaluations of the drift and diffusion fields. In contrast to the full IR-spectrum fitting pipeline of \citet{hanRefiningPotentialEnergy2025}, our benchmark optimizes a differentiable proxy for the spectral objective: the normalized squared dipole-velocity signal accumulated along finite-temperature Langevin trajectories.

For each batch $b$, we simulate paired trajectories from velocities $\pm v_0$ and minimize the empirical mean
\begin{equation}
    \label{eq:md_loss}
    \mathcal{L}_{\mathrm{MD}}^{(B)}(\theta)
    =
    \frac{1}{B}
    \sum_{b=1}^{B}
    \frac{1}{T\,N_{\mathrm{mol}}}
    \int_0^T
    \left\|\dot{\mu}_{\theta}^{(b)}(t)\right\|_2^2\,\mathrm{d}t .
\end{equation}
Here $\dot{\mu}_{\theta}$ is a charge-weighted molecular dipole-velocity proxy computed along the simulated trajectory, so the experiment retains the main computational bottleneck of spectral fitting: differentiating a neural force field through long stochastic rollouts of large state vectors.

\begin{table}[ht]
\centering
    \caption{Metrics for molecular dynamics. The step size is chosen so that each solver uses 252 evaluations per integration over the same rollout time. MCF midpoint became unstable.}
    \renewcommand{\arraystretch}{1.1}
    \begin{tabular}{lcccc}
        \toprule
        Method & \# Eval. / Step & Step Size & Terminal MSE ($\times 10^{-2}$) & Runtime (s) \\
        \midrule
        Reversible Heun & 1 & $1/2520$ & $44.85{\scriptstyle \pm 0.90}$ & $1083.5$ \\
        MCF Euler & 2 & $1/1260$ & $44.66{\scriptstyle \pm 0.42}$ & $984.1$ \\
        MCF Midpoint & 4 & $1/630$ & -- & $757.6$ \\
        $\mathrm{EES}(2,5)$ & 3 & $1/840$ & $45.04{\scriptstyle \pm 1.09}$ & $577.7$ \\
        \bottomrule
    \end{tabular}
\end{table}

\begin{wrapfigure}{R}{0.42\textwidth}
    \vspace{-1em}
    \includegraphics[width = 0.4\textwidth]{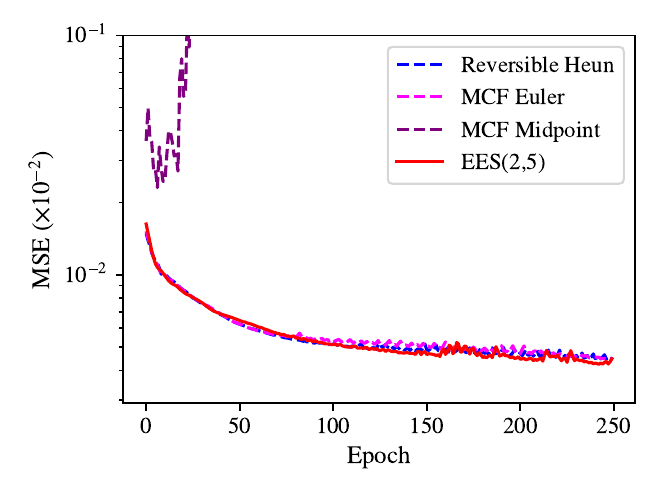}
    \captionsetup{oneside,margin={0.5cm,0cm}}
    \caption{\textit{Training proxy MSE for Langevin MD.}}
    \label{fig:md_mse}
\vspace{-3em}
\end{wrapfigure}
Under this fixed-budget comparison, $\mathrm{EES}(2,5)$ attains a terminal proxy error statistically indistinguishable from the stable baselines while reducing wall-clock time by roughly 47\% relative to Reversible Heun and 41\% relative to MCF Euler. MCF Midpoint diverged at the prescribed step size and is therefore omitted from the accuracy comparison. We note one limitation: the present subsection evaluates a differentiable Langevin proxy rather than the full experimental IR spectral loss, and full-spectrum fitting together with memory-scaling measurements should be reported separately before any stronger end-to-end IR-fitting claim is made.

\section{Experimental and Implementation Details}
\label{app:experimental_and_implementation_details}

\subsection{Hardware and Software Details}
\label{app:hardware_and_software_details}
\paragraph{Software details.} Experiments use Python 3.13 with CUDA 13.1. Models are implemented in \href{https://github.com/jax-ml/jax}{JAX} 0.10.0 \cite{bradbury_jax_2018} and \href{https://github.com/patrick-kidger/equinox}{Equinox} 0.13.7 \cite{kidger_equinox_2021}. Differential equation solves use \href{https://github.com/patrick-kidger/diffrax}{Diffrax} 0.7.1 \cite{kidger_neural_2021}, dataloading uses \href{https://github.com/luke-a-thompson/cyreal_dynamics}{Cyreal} 0.2.1 \cite{morad_cyreal_2026}, and path-signature losses use \href{https://github.com/luke-a-thompson/Stochastax}{Stochastax} 0.5.0 and \href{https://pypi.org/project/pysiglib/}{Pysiglib} 3.0.0 \cite{shmelev_pysiglib_2025} in the manifold and Euclidean cases, respectively. The implementation of the Reversible adjoint used in our experiments is sourced from Sam McCallum's \href{https://github.com/sammccallum/diffrax}{Diffrax fork} (\texttt{ReversibleAdjoint}).

The writing of this work required the creation of three software packages to support EES, $2N$, and geometric numerical integration in Diffrax and Julia. We present these packages here and the license under which they are released.
\begin{table}[ht]
    \centering
    \caption{New packages introduced to support the present work}
    \label{tab:new_packages}
    \begin{tabular}{llll}
        \toprule
        \textbf{Package} & \textbf{Ecosystem} & \textbf{Functionality} & \textbf{License} \\
        \midrule
        \href{https://anonymous.4open.science/r/2ndiffrax/README.md}{Diffrax-lowstorage}
        & Jax
        & $2N$ integrators, incl. $\mathrm{EES}$
        & \texttt{Apache-2.0} \\
        \href{https://anonymous.4open.science/r/jax_geo_int/README.md}{Georax}
        & Jax
        & Geometric integrators, incl. $\mathrm{CF\text{-}EES}$
        & \texttt{Apache-2.0} \\
        \href{https://anonymous.4open.science/r/EffectivelySymmetric-jl/}{EffectivelySymmetric.jl}
        & SciML
        & Standard, $2N$, and $\mathrm{CF\text{-}EES}$ methods
        & \texttt{Apache-2.0} \\
        \bottomrule
    \end{tabular}
\end{table}

\paragraph{Hardware details.}
All training and evaluation runs were conducted on a single workstation equipped with one NVIDIA RTX 5080 GPU with 16GB of GPU memory, an AMD Ryzen 9 9950X3D CPU, and 64GB of system memory, running Ubuntu 24.04 LTS.

\subsection{Additional Details for OU Experiments}
\label{app:ou_details}

\paragraph{Dynamics.}
We learn the Ornstein--Uhlenbeck (OU) dynamics
\begin{equation*}
    dy_t = \nu(\mu - y_t)\, dt + \sigma\, dW_t, \quad y_0 \in \mathbb{R},
\end{equation*}
under the high-volatility regime $\nu = 0.2$, $\mu = 0.1$, $\sigma = 2$.

\paragraph{Model.}
We use a Neural Langevin SDE (LSDE)~\cite{oh2024stable},
\begin{equation*}
    dz_t = g(z_t; \theta_g)\, dt + f(t; \theta_f) \circ dW_t, \quad z_0 = h(\mathbf{x}, \theta_h) \in \mathbb{R}^{d_z},
\end{equation*}
where $h$ is a learnable affine function of the input data $\mathbf{x} = \{x_n\}_{n\geq 0}$, $x_n \in \mathbb{R}^2$, sampled from the true OU dynamics, and $g, f$ are neural networks parametrised by $\theta_g, \theta_f$ respectively. We take latent dimensionality $d_z = 32$, and parametrise $f, g$ as 2-layer neural networks of width $32$ with LipSwish activations.

\paragraph{Training.}
The SDEs are integrated over $t \in [0,10]$ and the LSDE is trained for $250$ epochs using the Adam optimiser with a fixed learning rate of $10^{-3}$. At each epoch, $50{,}000$ realisations of the trained dynamics are sampled and the MSE loss is computed against the true OU dynamics.

\subsection{Additional Details for GBM Experiments}
\label{app:gbm_details}

\paragraph{Dynamics.}
We learn the dynamics of a high-dimensional geometric Brownian motion (GBM)
\begin{equation*}
    dy_t = A y_t\, dt + \sigma y_t\, dW_t, \quad y_0 \in \mathbb{R}^d,
\end{equation*}
where $A \in \mathbb{R}^{d \times d}$ and $\sigma \in \mathbb{R}$. We introduce a stiff drift component by choosing $A = QDQ^T$, where $D = \mathrm{diag}(\lambda_0, \lambda_1, \ldots, \lambda_{d-1})$ with $\lambda_i = -20\,(1 + \tfrac{i}{d})$ and $Q$ a randomly generated orthogonal matrix; we take $d = 25$ and $\sigma = 0.1$.

\paragraph{Model.}
We use a Neural Langevin SDE (LSDE) \cite{oh2024stable},
\begin{equation}
    \label{eq:neural_sde}
    dz_t = g(z_t; \theta_g)\, dt + f(z_t; \theta_f) \circ dW_t, \quad z_0 = h(\mathbf{x}, \theta_h) \in \mathbb{R}^{d_z},
\end{equation}
where $h$ is a learnable affine function of the input data $\mathbf{x} = \{x_n\}_{n\geq 0}$, $x_n \in \mathbb{R}^2$, sampled from the true dynamics, and $g, f$ are neural networks parametrised by $\theta_g, \theta_f$ respectively.

\paragraph{Training.}
We integrate the SDEs over $t \in [0,1]$ for $1{,}000$ epochs, sampling $10{,}000$ realisations of the dynamics at every epoch. The Adam optimiser is used with a fixed learning rate of $2 \times 10^{-1}$.

\subsection{Additional Details for Stochastic Volatility Experiments}
\paragraph{Model.}
We use an unconditional Euclidean neural SDE $\mathrm{d}x_\tau=f(x_\tau,\tau)\mathrm{d}\tau+g(x_\tau,\tau)\mathrm{d}W_\tau$, $x_0=\mathbf{1}$, for the stochastic-volatility benchmarks. The drift is a 4-layer MLP (width 16, LipSwish), while the diffusion is a 3-layer MLP (width 16, LipSwish, softplus output scaled by $0.2$). Both take $(\tau,x_\tau)\in\mathbb{R}^{1+d}$ as input, and $g$ is interpreted as the diagonal matrix $\mathrm{diag}(\sigma(x_\tau,\tau))$, giving learned state- and time-dependent coordinatewise volatility.

\paragraph{Dataset.}
We embed all stochastic volatility models in the RDE framework of \citet{bonesini_rough_2024}, which spans classical Black–Scholes \cite{Black1973} to contemporary rough volatility models \cite{gatheral_volatility_2014}. Numerically, we form the lead--lag path on a uniform noise grid of \(N_{\mathrm{noise}} = 128\) intervals and integrate the Wong--Zakai ODE on a finer uniform grid of \(N_{\mathrm{RDE}} = 768\) fixed steps with linear interpolation using Tsitouras' 5(4) Runge--Kutta method \cite{tsitouras_rungekutta_2011}, recording states at the noise times. The underlying Riemann-Liouville volatility process is simulated with the hybrid method of \cite{bennedsen_hybrid_2017}. Parameter selections are shown in Table~\ref{tab:rough-volatility-parameters}. We adopt the rBergomi parameter values of \citet{callum_rough_2023}, namely \(v_0 = 0.04\), \(\nu = 1.991\), \(H=0.25\), and \(\rho=-0.848\). For SPX options on 30/05/2022, this calibration gives a model-implied volatility surface with mean-squared error \(3.73\times 10^{-5}\) relative to market implied volatilities. The remaining coefficients are fixed to standard benchmark values in plausible financial ranges. Since we use the model in its standard form, we refer to \cite[Equation~1.2]{bonesini_rough_2024} for the full specification.

\begin{table}[ht]
\centering
\caption{Parameters for the rough volatility model configurations. Missing entries indicate parameters not used by the model.}
\label{tab:rough-volatility-parameters}
\begin{tabular}{@{}lccccccc@{}}
\toprule
Model & $S_0$ & $v_0$ & $\rho$ & $\nu$ & $H$ & $\lambda$ & $\bar{v}$ \\
\midrule
Black--Scholes & 1.0 & 0.04 & -- & -- & -- & -- & -- \\
Classical Bergomi & 1.0 & 0.04 & -0.7 & -- & -- & -- & -- \\
Local stoch volatility & 1.0 & 0.04 & -0.3 & -- & -- & 1.0 & 0.04 \\
Heston & 1.0 & 0.04 & -0.7 & 0.5 & -- & 1.5 & 0.04 \\
Rough Heston & 1.0 & 0.04 & -0.7 & 0.5 & 0.1 & 1.5 & 0.04 \\
Quadratic rough Heston & 1.0 & 0.04 & -- & -- & 0.1 & 1.0 & -- \\
Rough Bergomi & 1.0 & 0.04 & -0.848 & 1.991 & 0.1 & -- & -- \\
\bottomrule
\end{tabular}
\end{table}

\paragraph{Training.}
The model is trained with truncated (time-augmented) path-signature MMD$^2$ loss, vanilla SGD at learning rate $10^{-3}$, batch size $4{,}096$, 100~epochs, no schedule or gradient clipping at a fixed NFE budget of $504$. Test-time paths are single-rollout trajectories from $x_0 = \mathbf{1}$, evaluated via two-sample KS statistic at step $t = 55$.

\subsection{Additional Details for the Stochastic Kuramoto Experiments}
\label{app:torus_details}

\paragraph{Data-generating dynamics.}
Trajectories are simulated from equation~\eqref{eq:kuramoto_2nd_order} with bimodal natural frequencies $\Omega_i \in \{+P, -P\}$ \citep{filatrella2008analysis}, default parameters $m = 1$, $K = 2.0$, $P = 0.5$, $D = 0.05$. The deterministic $N = 2$ subsystem admits a unique stable phase-locked equilibrium at $\Delta\theta_\infty = \arcsin(2P/K)$ for $K > 2P$ \citep{acebron2005kuramoto}, used as a verification anchor for the simulator (Appendix below). The integration horizon is $T = 5\,\text{s}$, sub-sampled to $n_\text{obs} = 200$ uniform observation timepoints from a fine integration grid of $n_\text{fine} = 16{,}384$ steps using diffrax's \texttt{Heun} solver. Splits are 5{,}000 / 1{,}000 / 1{,}000 trajectories (train / val / test) per network size $N \in \{2, 4, 8\}$, each with an independent initial condition and Brownian path.

\paragraph{Simulator verification.}
At $\sigma = 0$, $N = 2$, the deterministic 2-oscillator phase difference at $T = 5\,\text{s}$ is independent of $n_\text{fine}$ for $n_\text{fine} \ge 1024$ (relative variation $<10^{-4}$ across $n_\text{fine} \in \{1024, 2048, 4096, 8192, 16384\}$), confirming Heun convergence at the chosen production grid. Residual $\sim 7.7\%$ deviation from $\arcsin(2P/K) = \pi/6$ is consistent with the analytical decay constant of the linearised dynamics around the phase-locked equilibrium ($\sim 0.5\,\text{s}^{-1}$) at finite $T$ rather than integrator error. At $\sigma > 0$, the per-trajectory order parameter $r(t) = \big| N^{-1}\sum_j e^{i\theta_j} \big|$ saturates to a stationary mean $r_\infty \approx 0.70 \pm 0.24$ over an ensemble of 128 trajectories, placing the operating point in the partial-synchronisation regime as designed.

\paragraph{Model.}
The neural SDE on $T\mathbb{T}^N = \mathbb{T}^N \times \mathbb{R}^N$ uses MLP drift and diffusion fields of width $128$ over the feature embedding $(\sin\theta, \cos\theta, \omega) \in \mathbb{R}^{3N}$, with depths $3$ (drift) and $2$ (diffusion), SiLU activation, and a softplus-output diffusion scaled by $0.1$. The diffusion is decoupled (additive noise on $\omega$ only). Integration uses CFEES(2,5) on the product Lie group $T\mathbb{T}^N$, lifted via Bazavov's commutator-free construction \citep{bazavov_commutator-free_2020}.

\paragraph{Loss.}
Multi-horizon energy score evaluated at horizons $h \in \{T/8, T/4, T/2, T\}$ with $m = 4$ Monte Carlo rollouts per horizon, using the wrapped-on-$\theta$, plain-on-$\omega$ distance $d((\theta_a,\omega_a), (\theta_b,\omega_b)) = \sum_i | \mathrm{wrap}(\theta_a^i - \theta_b^i) | + \sum_i | \omega_a^i - \omega_b^i |$. Optimisation uses AdamW at peak learning rate $10^{-3}$, gradient clipping at global norm $1$, $30$ epochs, batch size $64$.

\paragraph{Adjoint pilot (Table~\ref{tab:kuramoto_pilot}).}
Before any training sweep, we verify that the three adjoints under test compute mathematically equivalent gradients. Using a fixed (model, data, Brownian-path) cell at $N = 2$, $n_\text{steps} \in \{200, 1000, 5000\}$, batch $32$, we compute the relative $\ell_2$ distance between each adjoint's gradient and a fine-$dt$ reference at $n_\text{steps,ref} = 10{,}000$ under CFEES(2,5) with the Reversible adjoint. Differences across adjoints at the same $n_\text{steps}$ are at the float32 round-off floor; the residual gap to the fine-$dt$ reference reflects the discretisation difference between the test cell and reference grid (and uses an independent Brownian path), not adjoint error.

\begin{table}[ht]
\centering
\small
\caption{M3.1 + M3.3 gradient fidelity. Values are relative $\ell_2$ distance to a fine-$dt$ CFEES(2,5) with the Reversible adjoint reference at $n_\text{steps,ref} = 10{,}000$. The three adjoints agree to 4 significant figures at every $n_\text{steps}$ (the reported gap is shared discretisation error vs the fine grid).}
\label{tab:kuramoto_pilot}
\renewcommand{\arraystretch}{1.1}
\begin{tabular}{l ccc}
\toprule
$n_\text{steps}$ & Reversible & Full & Recursive \\
\midrule
$200$  & $9.251\times10^{-2}$ & $9.251\times10^{-2}$ & $9.251\times10^{-2}$ \\
$1000$ & $9.290\times10^{-2}$ & $9.290\times10^{-2}$ & $9.290\times10^{-2}$ \\
$5000$ & $9.378\times10^{-2}$ & $9.378\times10^{-2}$ & $9.378\times10^{-2}$ \\
\bottomrule
\end{tabular}
\end{table}

\paragraph{Memory scaling (Figure~\ref{fig:kuramoto_memory_scaling}).}
Table~\ref{tab:kuramoto_memory_data} gives the absolute peak GPU memory measurements that underlie Figure~\ref{fig:kuramoto_memory_scaling}: one forward$+$backward solve through the Kuramoto neural SDE at $N=1000$, batch $64$, hidden width $128$, on a single GPU, with each cell run in an isolated subprocess.

\begin{table}[ht]
\centering
\small
\caption{Peak GPU memory (MiB) for one forward$+$backward solve of the Kuramoto NSDE, $N=1000$, batch $64$.}
\label{tab:kuramoto_memory_data}
\renewcommand{\arraystretch}{1.1}
\begin{tabular}{rccc}
\toprule
$n_{\mathrm{steps}}$ & CF-EES(2,5) (Reversible) & CG2 (Full) & CG2 (Recursive) \\
\midrule
50      & 356 & 670     & 648 \\
100     & 356 & 696     & 650 \\
200     & 352 & 743     & 653 \\
500     & 355 & 891     & 658 \\
1{,}000  & 355 & 1{,}134  & 663 \\
2{,}000  & 355 & 2{,}040  & 674 \\
5{,}000  & 355 & 4{,}970  & 690 \\
10{,}000 & 352 & OOM     & 712 \\
\bottomrule
\end{tabular}
\end{table}

\subsection{Additional Details for Sphere Latent SDE Experiments}
\label{app:sphere_details}

\paragraph{Setting.}
\citet{zeng2023latent} parametrise the drift of the latent SDE as a Chebyshev polynomial in $t$ of an MLP applied to the encoder output, and integrate the resulting Stratonovich SDE with one-step \emph{geometric Euler--Maruyama}: at every step a Lie-algebra increment $\omega = K(t)\,\mathrm{d}t + \sigma\,\mathrm{d}W \in \mathfrak{so}(n)$ is formed and lifted to $\mathrm{SO}(n)$ via $\exp(\omega)$, which then acts on the current state $z_t \in S^{n-1}$ by left multiplication. Their ELBO combines reconstruction at observed times with a closed-form Girsanov path-KL on the sphere; both terms reparameterise through every integrator step, so backpropagation must traverse the entire trajectory---the regime in which a reversible adjoint pays off.

\paragraph{Backbone.}
The standard backbone setup for the UCI Human Activity benchmark of \citet{zeng2023latent} uses $z \in S^{15}$, batch $64$, and $N = 228$ integration steps per trajectory. We replace the geometric Euler--Maruyama step with $\mathrm{CF\text{-}EES}(2,5)$ and route backpropagation through the Reversible adjoint; the encoder, Chebyshev drift, decoder, and classification head are left unchanged.

\paragraph{Memory.}
Figure~\ref{fig:sphere_memory} reports peak GPU memory for one forward$+$backward through the integrator alone. The matrix exponential on $\mathfrak{so}(16)$ is heavier per step than the elementwise lift on $\mathbb{T}^d$, so the memory ratio at a given $N$ is larger than for the Kuramoto benchmark (Figure~\ref{fig:kuramoto_memory_scaling}). Table~\ref{tab:sphere_memory_data} gives the underlying measurements.

\begin{table}[ht]
\centering
\small
\caption{Peak GPU memory (MiB) for one forward$+$backward solve of the latent SDE on $S^{15}$, batch $64$. The reversible measurement at $n_{\mathrm{steps}}=5{,}000$ was not collected; the figure extrapolates from the constant trend at lower step counts.}
\label{tab:sphere_memory_data}
\renewcommand{\arraystretch}{1.1}
\begin{tabular}{rcc}
\toprule
$n_{\mathrm{steps}}$ & CF-EES(2,5) (Reversible) & Geometric Euler (Full) \\
\midrule
50    & 22 & 212    \\
200   & 23 & 802    \\
800   & 28 & 3{,}166 \\
2{,}000 & 40 & 7{,}894 \\
5{,}000 & -- & 19{,}711 \\
\bottomrule
\end{tabular}
\end{table}

\paragraph{Compute parity.}
Table~\ref{tab:sphere_parity} reports test accuracy after $30$ training epochs at a fixed total NN-evaluation budget per trajectory $B = 30$ (geometric Euler--Maruyama uses $N = B$ steps; $\mathrm{CF\text{-}EES}(2,5)$ uses $N = B/3$), averaged over two seeds. The published 990-epoch accuracy of \citet{zeng2023latent} at the full $N = 228$ grid is $90.56 \pm 0.45\%$; both solvers in our 30-epoch parity sweep land within the expected range for the truncated training and reduced step count.

\subsection{Additional Details for Molecular Dynamics Experiment}
\label{app:md_details}
\paragraph{Model.}
The EANN potential uses $n_{\mathrm{GTO}} = 12$ radial Gaussian-type orbitals per species pair with a $6$~\AA{} cutoff, followed by a per-element MLP of two hidden layers of width $64$ with SiLU activations and LayerNorm \cite{zhangEmbeddedAtomNeural2019}; pre-trained weights are loaded from \texttt{params\_eann4.pickle}. Forces are obtained by automatic differentiation of the total energy. The Langevin SDE is integrated with friction $\gamma = 1.0~\mathrm{ps}^{-1}$ and fluctuation--dissipation noise $\sigma = \sqrt{2\gamma k_{\mathrm{B}} T / \mathbf{m}}$ at $T = 298.15~\mathrm{K}$, with all quantities expressed in GROMACS units (nm, ps, g/mol). The dipole-velocity proxy uses charge weights $w_{\mathrm{O}} = 1,\, w_{\mathrm{H}} = -1/2$ and is accumulated as an augmented coordinate of the integrator state so its gradient flows through the full rollout.

\paragraph{Dataset.}
Initial configurations are drawn from the bundled \texttt{water64.pdb} structure (cubic periodic box, edge $1.86~\mathrm{nm}$). Per training step we sample a batch of $6$ initial conditions: positions are the reference geometry perturbed by $\mathcal{N}(\mathbf{0}, 10^{-3}~\mathrm{nm})$ and velocities are drawn from the Maxwell--Boltzmann distribution at $T$. Each batch element is duplicated under time reversal $\mathbf{v}_0 \mapsto -\mathbf{v}_0$, giving $12$ effective trajectories. Pair lists are rebuilt at every force evaluation with a $0.6~\mathrm{nm}$ neighbour-list cutoff.

\paragraph{Training.}
We rollout for $t_{\mathrm{end}} = 0.1~\mathrm{ps}$ and train for $250$ optimization steps with Adam at learning rate $5\times 10^{-4}$, gradients clipped to global norm $1.0$. All non-reversible base solvers (Euler, Midpoint) are lifted to the reversible setting via the algebraically reversible wrapper of \citet{mccallum2024efficient} with coupling parameter $\lambda = 0.999$, and every solver uses the Reversible adjoint. The reported terminal proxy MSE is the loss of \eqref{eq:md_loss} averaged over the batch at the final training step; runtime is wall-clock time for the full $250$-step run on a single GPU with fixed random seed across solvers.

\subsection{Memory Benchmark for Figure~\ref{fig:torus_scaling}}
\label{app:torus_scaling_data}

Figure~\ref{fig:torus_scaling} reports peak XLA scratch memory (\texttt{temp\_bytes}) for one forward$+$backward solve of an SDE on the $7$-torus $\mathbb{T}^7$, batch $1{,}024$, hidden width $128$. Each cell is run in an isolated subprocess. Table~\ref{tab:torus_intro_memory_data} gives the underlying measurements in MiB.

\begin{table}[ht]
\centering
\small
\caption{Peak XLA scratch memory (MiB) for one forward$+$backward solve of an SDE on $\mathbb{T}^7$, batch $1{,}024$.}
\label{tab:torus_intro_memory_data}
\renewcommand{\arraystretch}{1.1}
\begin{tabular}{rccccc}
\toprule
$n_{\mathrm{steps}}$ & CF-EES (Reversible) & CG2 (Full) & CG2 (Recursive) & CG4 (Full) & CG4 (Recursive) \\
\midrule
5      & 390 & 384 & 673 & 402 & 708 \\
10     & 391 & 385 & 673 & 403 & 708 \\
20     & 391 & 385 & 673 & 404 & 708 \\
50     & 390 & 386 & 673 & 405 & 708 \\
100    & 390 & 387 & 673 & 406 & 708 \\
200    & 390 & 390 & 673 & 408 & 708 \\
400    & 390 & 395 & 673 & 413 & 708 \\
800    & 391 & 406 & 674 & 425 & 709 \\
2{,}000  & 390 & 439 & 674 & 458 & 709 \\
5{,}000  & 391 & 522 & 675 & 540 & 710 \\
10{,}000 & 390 & 658 & 676 & 676 & 712 \\
\bottomrule
\end{tabular}
\end{table}

\end{document}